\documentclass[11pt]{article}

\usepackage{fullpage}

\usepackage[utf8]{inputenc} 
\usepackage[T1]{fontenc}    
\usepackage[colorlinks=true, citecolor=darkblue, linkcolor=darkblue, urlcolor=darkblue]{hyperref}       
\usepackage{url}            
\usepackage{booktabs}       
\usepackage{amsfonts}       
\usepackage{nicefrac}       
\usepackage{microtype}      
\usepackage[table, x11names]{xcolor}         
\definecolor{darkblue}{RGB}{0,0,139}
\usepackage{bm, bbm, amsthm, amsmath,amsfonts,amssymb, dsfont}
\usepackage{algorithm}
\usepackage{algpseudocode}
\usepackage{cleveref}
\usepackage{kotex}
\usepackage{graphicx}
\usepackage{subcaption}
\usepackage{tikz}
\usepackage{natbib}
\usetikzlibrary{arrows.meta, positioning}

\newtheorem{theorem}{Theorem}
\newtheorem{lemma}{Lemma}

\newtheorem{corollary}{Corollary}

\newtheorem{assumption}{Assumption}

\newcommand{\calA}{{\mathcal{A}}}

\newcommand{\calF}{{\mathcal{F}}}

\newcommand{\calL}{{\mathcal{L}}}
\newcommand{\calM}{{\mathcal{M}}}
\newcommand{\calN}{{\mathcal{N}}}
\newcommand{\calO}{{\mathcal{O}}}
\newcommand{\calP}{{\mathcal{P}}}

\newcommand{\calS}{{\mathcal{S}}}

\newcommand{\calV}{{\mathcal{V}}}

\newcommand{\calX}{{\mathcal{X}}}

\newcommand{\bbE}{\mathbb{E}}

\newcommand{\bbN}{\mathbb{N}}

\newcommand{\bbP}{\mathbb{P}}

\newcommand{\bbR}{\mathbb{R}}

\newcommand{\bfw}{\mathbf{w}}
\newcommand{\bfx}{\mathbf{x}}

\newcommand{\Regret}{\mathrm{Regret}}
\newcommand{\Violation}{\mathrm{Violation}}

\newcommand{\spn}{\mathrm{sp}}

\DeclareMathOperator*{\argmax}{arg\,max}
\DeclareMathOperator*{\argmin}{arg\,min}
\newcommand{\relint}{\mathrm{relint}}
\newcommand{\ext}{\mathrm{ext}}

\newcommand{\conv}{\mathrm{conv}}
\newcommand{\clconv}{\overline{\mathrm{conv}}}
\newcommand{\bigO}{\widetilde{{\mathcal{O}}}}
\newcommand{\PiSR}{{\Pi}}
\newcommand{\PiSD}{{\Pi^{\textnormal{SD}}}}

\newcommand{\calPSR}[1]{\widehat\calP({#1})}
\newcommand{\calPSD}[1]{\widehat\calP^{\textnormal{SD}}({#1})}
\newcommand{\slater}{\gamma}
\newcommand{\occ}[2]{q_{#2}^{#1}}

\newcommand{\bmSigma}{\bm{\Sigma}}
\input{arXiv/style}

\title{
Learning Weakly Communicating Average-Reward CMDPs: Strong Duality and Improved Regret
}

\author{
Kihyun Yu$^{1}$ \quad Beomhan Baek$^{2}$ \quad Dabeen Lee$^{2}$\\[0.3em]
$^{1}$Department of Industrial \& Systems Engineering, KAIST\\
$^{2}$Department of Mathematical Sciences, Seoul National University\\
\texttt{khyu99@kaist.ac.kr, \{bhbaek2001,dabeenl\}@snu.ac.kr}
}
\author{
Kihyun Yu$^{1}$ \quad Beomhan Baek$^{2}$ \quad Dabeen Lee$^{2}$\\[0.3em]
$^{1}$KAIST \quad $^{2}$Seoul National University\\
\texttt{khyu99@kaist.ac.kr, \{bhbaek2001,dabeenl\}@snu.ac.kr}
}

\date{}

\begin{document}
\maketitle
\begin{abstract}
We study infinite-horizon average-reward constrained Markov decision processes (CMDPs) under the weakly communicating assumption. Our contributions are twofold. First, we establish strong duality for weakly communicating average-reward CMDPs over stationary policies with finite state and action spaces. Despite the absence of a linear programming formulation and the resulting nonconvexity under the weakly communicating setting, we show that strong duality still holds by carefully exploiting the geometric structure of the occupation measure set. Second, building on this result, we propose a primal--dual clipped value iteration algorithm for learning weakly communicating average-reward linear CMDPs. Our algorithm achieves regret and constraint violation bounds of $\widetilde{\mathcal{O}}(T^{2/3})$, improving upon the best known bounds, where $T$ denotes the number of interactions. Our approach extends clipped value iteration to the constrained setting and adapts it to a finite-horizon approximation, which stabilizes the dual variable and is crucial for achieving improved regret bounds. To analyze this, we develop a novel approach based on strong duality that enables the decomposition of the composite Lagrangian regret into separate bounds on regret and constraint violation.
\end{abstract}

\section{Introduction}
Safe reinforcement learning (RL) studies how to optimize performance while ensuring that safety constraints are satisfied during learning~\citep{garcia2015comprehensive}. It is crucial in safety-critical applications such as autonomous driving, where unsafe actions may lead to collisions~\citep{driving-isele2018safe}; robotics, where constraint violations can damage the environment~\citep{robotics-achiam2017constrained}; healthcare, where unreliable decisions may harm patients~\citep{healthcare-coronato2020reinforcement}; and large language models, where unsafe outputs must be avoided~\citep{dai2023safe}. These problems can be naturally formulated as constrained Markov decision processes (CMDPs)~\citep{altman1999constrained}.

To capture long-run optimal behavior, infinite-horizon average-reward CMDPs have been widely studied. In this setting, the long-run average is defined as the limit of the expected time-average of the reward and constraint functions, and the objective is to find a policy that maximizes the long-run average reward while ensuring that the long-run average constraint satisfies a prescribed threshold. Accordingly, a number of learning algorithms have been developed for this problem~\citep{wei2022provably, singh2022learning, chen2022learning, ghosh2023achieving}.

A key feature of learning average-reward CMDPs is that both the difficulty of the problem and the resulting performance guarantees depend critically on the structural assumptions of the transition dynamics, such as ergodic, unichain, and weakly communicating assumptions~\citep{puterman1994markov}.\footnote{If an MDP is ergodic, then it is unichain; if it is unichain, then it is weakly communicating. For more details on MDP classifications, see \Cref{appendix:additional preliminary}.} For instance, in the tabular setting, \citet{chen2022learning} proposed computationally efficient algorithms under different assumptions, achieving regret bounds of $\bigO(\sqrt{T})$ under the ergodic assumption and $\bigO(T^{2/3})$ under the weakly communicating assumption, while also providing a computationally inefficient algorithm achieving $\bigO(\sqrt{T})$ under the weakly communicating assumption, where $T$ denotes the number of interactions. More recently, under a similar assumption, \citet{ghosh2023achieving} extended the algorithm to the linear setting, establishing a regret bound of $\bigO(T^{3/4})$. These examples illustrate that weaker structural assumptions, while more general, lead to increased analytical challenges and weaker performance guarantees.

The fundamental reason for this gap is that weaker structural assumptions admit fewer technical tools. A prominent example is strong duality. Although it plays a central role in many CMDP algorithms~\citep{paternain2019constrained}, its applicability in average-reward CMDPs depends on the underlying structural assumptions. In particular, in the average-reward CMDP literature, strong duality has been established under the unichain assumption~\citep{altman1999constrained}. More recently, \citet{bai2024learning} established strong duality under general policy parameterization; however, their result relies on the ergodic assumption. In contrast, beyond these stronger structural assumptions, its validity under the weakly communicating assumption remains unclear. These examples highlight that weaker structural assumptions limit the available technical tools---such as strong duality---thereby motivating the need to extend these tools to enable improved learning algorithms.

Motivated by this limitation, we study weakly communicating average-reward CMDPs, starting with an investigation of strong duality under this assumption. The main challenge is that, unlike the unichain case, this problem does not admit an equivalent linear programming (LP) reformulation. Specifically, under the unichain assumption, the problem can be reformulated as a linear program (LP) via occupation measures~\citep{altman1999constrained}, allowing strong duality to be directly applied. In contrast, under the weakly communicating assumption, the occupation measure formulation no longer yields an LP, and the resulting problem becomes nonconvex. This fundamental difference significantly complicates the analysis and necessitates new techniques to establish strong duality in this more general setting. Despite these challenges, our main contributions are summarized as follows.
\begin{itemize}
    \item We establish strong duality for weakly communicating average-reward CMDPs. In particular, we consider the average-reward CMDP over stationary policies and show that, when the problem is feasible, an optimal dual variable exists and the primal and dual problems attain the same optimal value. The main challenge in establishing strong duality lies in the fact that the set of occupation measures does not form a polytope and is, in fact, nonconvex. To address this, we show that certain geometric properties---namely, that the set behaves like a polytope with some missing boundary points---are sufficient to invoke strong duality for LPs. To the best of our knowledge, this is the first strong duality result for weakly communicating average-reward CMDPs over stationary policies, and it admits broad applicability in the analysis of a wide class of learning algorithms.
    
    \item Building on our strong duality result, we propose a primal–dual clipped value iteration algorithm for learning weakly communicating average-reward linear CMDPs. Our algorithm achieves regret and constraint violation bounds of $\bigO(T^{2/3})$, improving upon the best known bound of $\bigO(T^{3/4})$~\citep{ghosh2023achieving}. To design the algorithm, we extend clipped value iteration~\citep{hong2024reinforcement} to the constrained setting and adapt it to the finite-horizon approximation, whereas prior work applies it to the discounted-reward approximation. The finite-horizon approximation stabilizes the dual variable, making it a crucial modification for achieving improved regret bounds in the constrained setting, together with value function clipping. To complete the analysis, we develop a novel approach that fundamentally relies on the strong duality result. This enables us to decompose the composite Lagrangian regret into separate bounds on regret and constraint violation.
\end{itemize}
Further details on related work can be found in the appendix.


\section{Preliminaries}\label{sec:preliminary}


\paragraph{Infinite-horizon average-reward CMDP}
Let $\calM = (\calS, \calA, P, r, g, b)$ denote a CMDP, where $\calS$, $\calA$ are the finite state and action spaces, $P$ is the transition kernel, $r,g:\calS\times\calA \to [0,1]$ are the reward and constraint functions, respectively, and $b\in[0,1]$ is the constraint threshold. In addition, $P(s'|s,a)$ represents the transition probability from state $s \in \calS$ to state $s' \in \calS$ by taking action $a\in \calA$. For simplicity, we assume that $r,g$ are deterministic. Given $\calM$ and initial state $s_1 \in \calS$, we define an average-reward CMDP as 
\begin{equation}\label{eq:average-reward CMDP}
    \sup_{\pi \in \PiSR} \ J_r^\pi(s_1) \quad\textnormal{s.t.}\quad J_g^\pi(s_1) \geq b,\tag{CMDP}
\end{equation}
where $J_r^\pi(s_1)$, $J_g^{\pi}(s_1)$ denote the gains associated with the policy $\pi$, $\PiSR$ denotes the set of stationary randomized policies, and $J_r^*(s_1)$ denotes the optimal value of \eqref{eq:average-reward CMDP}. We define $J_r^\pi(s) \triangleq \lim_{T\to \infty}\frac{1}{T}\bbE_{\pi}\big[\sum_{t=1}^T r(s_t,a_t)|s_1=s\big]$, where the expectation is taken over the trajectory $\{(s_t,a_t)\}_{t=1}^T$ generated by $P$ and $\pi$, with $s_{t+1} \sim P(\cdot|s_t,a_t)$ and  $a_t \sim \pi(\cdot|s_t)$ for all $t\in [T]$. 
We define the state bias functions as $v_r^\pi(s)\triangleq\lim_{T\to \infty}\frac{1}{T}\sum_{t=1}^T\bbE_{P,\pi}\left[\sum_{i=1}^t (r(s_i,a_i) - J_r^\pi(s_i)) | s_1=s\right]$. Note that $J_g^\pi(s)$ and $v_g^\pi(s)$ can be defined in similar ways. Given any bounded function $V: \calS\to \bbR$, we define the span of $V$ as $\spn(V) \triangleq \max_{s\in\calS}V(s) - \min_{s'\in\calS}V(s')$, and $PV(s,a) \triangleq \sum_{s'}P(s'|s,a)V(s')$.

\paragraph{Weakly communicating setting}
Throughout the paper, we assume weakly communicating CMDPs, defined as follows.
An MDP $\calM$ is said to be \emph{weakly communicating} if 
there exists a stationary randomized policy which induces a Markov chain with a single closed irreducible class and a set of states which is transient under all stationary policies (Proposition 8.3.1 in \citet{puterman1994markov}). For additional notions of Markov chains and MDP classifications, we refer the reader to \Cref{appendix:additional preliminary}.


We introduce the notion of occupation measure, which represents the limiting expected frequency of visiting the state--action pair $(s,a)$ under policy $\pi$ starting from $s_1$. Given an MDP and an initial state $s_1$, we define the \emph{occupation measure} associated with a stationary policy $\pi$ as $\occ{\pi}{s_1} \in \mathbb{R}^{|\calS|\times|\calA|}$. For each $(s,a) \in \calS \times \calA$, it is given by $\occ{\pi}{s_1}(s,a) \triangleq \lim_{T\to\infty}\frac{1}{T}\mathbb{E}_{\pi}\left[\sum_{t=1}^T \mathds{1}\{s_t=s,a_t=a\}\mid s_1\right]$.

For a stationary policy $\pi$, when the state and action spaces are finite, the above limit exists (Proposition 8.9.1 in \citep{puterman1994markov}), and $J_r^\pi(s_1)$ can be expressed as $r^\top \occ{\pi}{s_1}$, with an analogous expression for $g$. In the weakly communicating model, the occupation measure may depend on the initial state, which motivates explicitly including $s_1$ in the notation. Accordingly, let $\calPSR{s_1}$ denote the set of occupation measures achievable by stationary policies from $s_1$, defined as
\begin{equation*}
    \calPSR{s_1} \triangleq \left\{ \occ{\pi}{s_1}\in \bbR^{|\calS|\times|\calA|} : \pi \in \PiSR \right\}.
\end{equation*}
Using occupation measures, we rewrite \eqref{eq:average-reward CMDP} as follows. 
\begin{equation}\label{eq:occ average-reward CMDP}
    \sup_{q \in \calPSR{s_1}} \ r^\top q \quad\textnormal{s.t.}\quad g^\top q \geq b.\tag{Occ-CMDP}
\end{equation}

Additionally, we define relative interior and extreme points of a convex set as follows. Given a convex set $\calP\subseteq \bbR^n$, a point $x \in \calP$ is called \emph{relative interior} if for all $y \in \calP$, there exists $\epsilon > 1$ such that $\epsilon x + (1-\epsilon)y \in \calP$. A point $x \in \calP$ is called an \emph{extreme} point if $x = \epsilon y +(1-\epsilon)z$ implies that $x=y=z$ for $y,z \in \calP$ and $\epsilon \in (0,1)$. Let $\relint(\calP)$ and $\ext(\calP)$ denote the set of relative interior and extreme points of $\calP$, respectively. 

\paragraph{Online linear CMDP} 
We next formulate the online infinite-horizon average-reward linear CMDPs on which \Cref{sec:average-reward linear CMDP} mainly focuses. The agent learns an optimal policy for an unknown CMDP through interaction with the environment. The interaction starts from an arbitrary initial state $s_1 \in \calS$. At each step $t \in [T]$, the agent observes $s_t$ and takes an action $a_t \sim \pi_t(\cdot \mid s_t)$. The environment then samples $s_{t+1} \sim P(\cdot \mid s_t, a_t)$. The performance metrics are regret and constraint violation, defined as $\Regret(T) = \sum_{t=1}^T \left(J_r^*(s_1) - r(s_t,a_t)\right)$ and $\Violation(T) = \sum_{t=1}^T \left(b - g(s_t,a_t)\right)$, where $J_r^*(s_1)$ denotes the optimal value of \eqref{eq:average-reward CMDP}. 

A CMDP $\calM$ is called \emph{linear} if the following conditions hold. For a known feature mapping $\phi:\calS\times \calA \to \bbR^{d}$, the reward and constraint functions, and transition kernel admit the following linear representations: for some $\theta_r,\theta_g \in \bbR^d$ and $\psi:\calS \to \bbR_+^d$, $r(s,a) = \phi(s,a)^\top \theta_r$, $g(s,a) = \phi(s,a)^\top \theta_g$, and $P(s'|s,a) = \phi(s,a)^\top \psi(s')$ for all $(s,a,s') \in \calS \times\calA \times \calS$. We further assume that $\|\phi(s,a)\|_2 \leq 1$ for all $(s,a)\in\calS\times\calA$, and $\|\theta_r\|_2,\|\theta_g\|_2\leq \sqrt{d}, \|\sum_{s'\in\calS}\psi(s')\|_2\leq \sqrt{d}$. For simplicity, we assume that $\theta_r,\theta_g$ are known, while $\psi$ is unknown\footnote{Our results for the online algorithm can be extended to the case of unknown stochastic rewards and constraints without affecting the main contributions up to constant factors.}. 

We next introduce the Slater condition, a standard assumption for learning CMDPs~\citep{ghosh2023achieving}. While our strong duality theorem itself does not require the Slater condition, this condition guarantees an upper bound on the optimal dual variable (\Cref{lem:optimal dual variable bound}), which is useful for designing and analyzing our online algorithm.
\begin{assumption}[Slater condition]\label{assum:Slater condition}
    There exists a Slater policy $\bar\pi \in \PiSR$ such that $J_g^{\bar\pi}(s_1) \geq b+\slater$ for some Slater constant $\slater >0$. Moreover, we assume that $\slater$ is known, while $\bar\pi$ is unknown.
\end{assumption}

\section{Strong duality}\label{sec:strong duality}

In this section, we establish strong duality and highlight the challenges that arise, particularly in the weakly communicating setting. We begin by introducing the Lagrangian formulation. For a dual variable $\lambda \geq 0$, the Lagrangian associated with the primal problem~\eqref{eq:average-reward CMDP} is defined as
\begin{equation*}\label{eq:lagrangian}
\calL(\pi,\lambda) = J_r^\pi(s_1) + \lambda \big(J_g^\pi(s_1) - b\big).
\end{equation*}
The corresponding dual function is defined as the supremum of $\calL$ over stationary policies:
\begin{equation*}\label{eq:dual function}
D(\lambda) = \sup_{\pi \in \PiSR} \calL(\pi,\lambda).
\end{equation*}
The dual problem is then given by $\inf_{\lambda \geq 0} D(\lambda)$. Let $\lambda^*$ denote a minimizer of $D(\lambda)$, when it exists. We are now ready to state the strong duality result.
\begin{theorem}\label{thm:strong duality}
    Let $\calM = (\calS, \calA, P, r,g,b)$ be a weakly communicating CMDP with $|\calS|, |\calA| < \infty$, and let $s_1 \in \calS$ be any initial state. Suppose that \eqref{eq:average-reward CMDP} is feasible.
    Then the following properties hold:
    \begin{enumerate}
        \item There exists $\lambda^*\geq 0$ such that $D(\lambda^*) = \inf_{\lambda\geq 0}D(\lambda)$.
        \item Strong duality, i.e., $J_r^*(s_1) = D(\lambda^*)$.
    \end{enumerate}
\end{theorem}
Even when the underlying MDP is only weakly communicating, this theorem guarantees that, with finite state and action spaces and feasibility, the dual problem attains its minimum at some nonnegative dual variable. It further establishes strong duality, so the optimal values of the primal and dual problems coincide. This property is central in many analyses, and in \Cref{sec:average-reward linear CMDP}, we show how it can be leveraged to analyze an online algorithm and obtain improved regret upper bounds.

\subsection{Challenges and techniques: strong duality}
While strong duality for CMDPs may appear well established—suggesting that \Cref{thm:strong duality} should follow directly—this is not the case in our setting. The difficulty arises because we assume only a weakly communicating structure. In particular, under this assumption, \eqref{eq:average-reward CMDP} cannot be easily reformulated as an equivalent LP, making the application of LP strong duality nontrivial.

To clarify this challenge, we first explain why establishing strong duality becomes more straightforward under stronger structural assumptions. For instance, \cite{altman1999constrained} considers the setting where the MDP satisfies a unichain structure. In such a case, the problem admits an equivalent LP reformulation, and the corresponding duality follows directly from LP duality. Specifically, one can show that the set of occupation measures induced by stationary policies is indeed a polytope (Theorem 4.2 in \citet{altman1999constrained}), i.e., $\calPSR{s_1} = \calP$, where the polytope $\calP$ is defined as
\begin{align*}\label{eq:calP def}
    \calP\triangleq \left\{q\in \bbR^{|\calS|\times|\calA|}: 
    {\begin{aligned}
        &\sum_{s,a} q(s,a) =1, \ q(s,a) \geq 0 \ \forall  (s,a) \in \calS\times\calA,\\
        &\sum_{a}q(s,a) = \sum_{s',a'}q(s',a')P(s|s',a'), \ \forall s\in \calS
    \end{aligned}}
    \right\}.
\end{align*}
As a consequence, \eqref{eq:occ average-reward CMDP} admits the following equivalent LP reformulation, and this connection allows us to rely on LP strong duality under the unichain assumption:
\begin{equation}\label{eq:LP unichain}
\eqref{eq:occ average-reward CMDP}
\ \overset{\calPSR{s_1} = \calP}{=} \
\begin{aligned}
    \max_{q \in \calP}\; r^\top q \; \text{s.t.}\; g^\top q \geq b.
\end{aligned}
\end{equation}

However, this is not applicable under the weakly communicating assumption, as $\calPSR{s_1} \neq \calP$, and $\calPSR{s_1}$ may be nonconvex~\citep{mannor2005empirical}. This prevents an LP reformulation of the CMDP, in contrast to the unichain case as in \eqref{eq:LP unichain}. Consequently, even when expressed in terms of occupation measures, \eqref{eq:occ average-reward CMDP} remains a nonconvex problem, which constitutes the main obstacle to establishing strong duality, particularly in the weakly communicating setting.

To overcome these limitations in the weakly communicating setting, our key observation is that $\calPSR{s_1}$ still possesses several favorable geometric properties; perhaps surprisingly, these properties satisfy sufficient conditions for applying LP strong duality, even though $\calPSR{s_1} \neq \calP$ and may be nonconvex. These properties are formalized in the following lemma.\footnote{This observation was noted in Remark 3.2 of \cite{mannor2005empirical} without a rigorous proof. For completeness, we provide a proof in \Cref{appendix:Strong duality}.}
\begin{lemma}\label{lem:relint & ext}
    Let $\calM = (\calS, \calA, P, r)$ be a weakly communicating MDP with $|\calS|, |\calA| < \infty$, and let $s_1\in\calS$ be any initial state. We have $\relint(\calP) \subseteq \calPSR{s_1}$, $\ext(\calP) \subseteq \calPSR{s_1}$, and $\calPSR{s_1} \subseteq \calP$.
\end{lemma}

\begin{figure}[t]
    \centering
    \begin{subfigure}{0.25\textwidth}
        \centering
        \includegraphics[width=\linewidth]{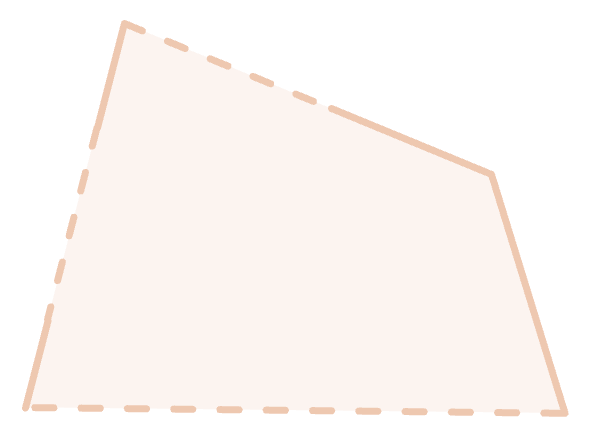}
    \end{subfigure}
    \hspace{2cm}
    \begin{subfigure}{0.25\textwidth}
        \centering
        \raisebox{-1mm}{\includegraphics[width=\linewidth]{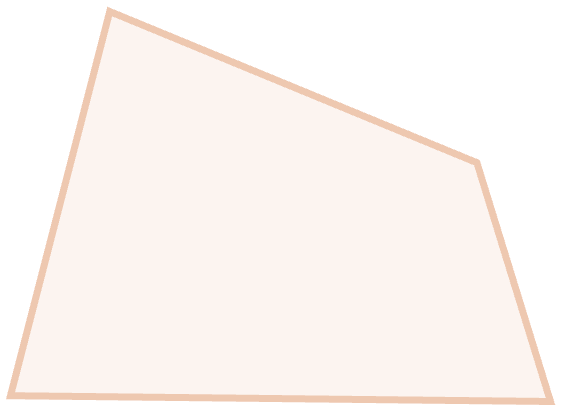}}
    \end{subfigure}
    \caption{Comparison of $\calPSR{s_1}$ and $\calP$ under the weakly communicating assumption.}
    \label{fig:calPSR}
\end{figure}
Intuitively, this lemma suggests that \eqref{eq:occ average-reward CMDP} may retain key properties of LPs, even though it is not itself an LP. This is because $\calPSR{s_1}$ has a structure similar to that of a polytope, in the sense that only certain boundary points are missing from $\calP$, as illustrated in \Cref{fig:calPSR}. In particular, $\relint(\calP) \subseteq \calPSR{s_1}$ in \Cref{lem:relint & ext} allows us to recover the same optimal value as follows:
\begin{equation*}\label{eq:LP weakly}
\eqref{eq:occ average-reward CMDP}
\ \overset{\relint(\calP) \subseteq \calPSR{s_1}}{=} \
\begin{aligned}
    \max_{q \in \calP}\; r^\top q \;
    \text{s.t.}\; g^\top q \geq b.
\end{aligned}
\end{equation*}
Thus, although $\calPSR{s_1} \neq \calP$, the two problems share the same optimal value, establishing a key connection between \eqref{eq:occ average-reward CMDP} and an LP over $\calP$. Now, since $\calP$ is a polytope, strong duality for LPs applies, allowing us to express the right-hand side equivalently as
\begin{equation*}
    \max_{q \in \calP}\; r^\top q \ \text{s.t.}\ g^\top q \geq b 
    \quad=\quad
    \min_{\lambda \geq 0} \max_{q \in \calP}\ r^\top q + \lambda (g^\top q - b).
\end{equation*}
Since $\calPSR{s_1} \neq \calP$ in our case, it remains to show that $\calP$ in the right-hand side can be replaced by $\calPSR{s_1}$ so as to recover the dual function $D(\lambda)$. This follows from the inclusion $\ext(\calP) \subseteq \calPSR{s_1}$. Specifically, for any linear objective, an LP over $\calP$ attains its optimum at an extreme point, which implies that the same optimal value is obtained when optimizing over $\calPSR{s_1}$. Namely, for any $\lambda \geq 0$,
\[
\max_{q \in \calP} \; r^\top q + \lambda(g^\top q - b)
\;=\;
\sup_{q \in \calPSR{s_1}} \; r^\top q + \lambda(g^\top q - b).
\]
Note that the right-hand side coincides with the dual function $D(\lambda)$ expressed in terms of occupation measures. Taking $\min_{\lambda\geq0}$ on both sides yields
\[
\min_{\lambda\geq0}\max_{q \in \calP} \; r^\top q + \lambda(g^\top q - b)
\;=\;
\min_{\lambda\geq 0} D(\lambda).
\]
Consequently, combining these, we obtain $\eqref{eq:average-reward CMDP}
= \min_{\lambda\geq 0}D(\lambda)$---strong duality for average-reward CMDP under the weakly communicating assumption.

\paragraph{Discussion on multichain CMDPs}
A natural further question is whether strong duality holds under more general structural assumptions, such as multichain CMDPs (see \Cref{appendix:additional preliminary} for the formal definition). This appears more challenging, since the key ingredient, \Cref{lem:relint & ext}, may no longer hold. For comparison, under the weakly communicating assumption, the inclusion $\relint(\calP)\subseteq \calPSR{s_1}$ heavily relies on the existence of a stationary policy that induces a single closed irreducible class, together with properties of the relative interior. However, such a policy is not guaranteed for multichain CMDPs. Therefore, the desired inclusion may fail, and so may $\ext(\calP) \subseteq \calPSR{s_1}$, making strong duality over stationary policies unclear.

\section{Improved algorithm for average-reward linear CMDPs}\label{sec:average-reward linear CMDP}
In this section, we propose an algorithm for learning weakly communicating average-reward linear CMDPs, establishing improved regret bounds. We emphasize that the improvement stems from a novel primal-dual algorithm, together with an analysis enabled by \Cref{thm:strong duality}.

\paragraph{Finite-horizon approximation}
As a preliminary, we first introduce a method for learning average-reward MDP problems---\emph{episodic finite-horizon approximation} framework~\citep{wei2021learning, chen2022learning, ghosh2023achieving}. This approach partitions the total iterations $T$ into $K$ episodes of length $H$, such that $T = HK$, and applies algorithms designed for finite-horizon CMDPs. Rather than directly estimating the average-reward quantities $J_r^\pi$ or $v_r^\pi$, the method focuses on learning $Q_{r,h}^\pi$ and $V_{r,h}^{\pi}$, referred to as $Q$-function and value function under finite-horizon approximation. For a fixed horizon $H$ and any reward function $r$, these functions are defined as $Q_{r,h}^\pi(s,a) = \bbE_\pi\left[\sum_{j=h}^H r(s_j,a_j)|s_h=s,a_h=a\right],\ V_{r,h}^\pi(s) = \bbE_\pi\left[\sum_{j=h}^H r(s_j,a_j)|s_h=s\right].$ The validity of the finite-horizon approximation is supported by the following lemma, which establishes a connection between two quantities, $V_{r,h}^\pi$ and $J_r^\pi$, arising from different settings.
\begin{lemma}
[Lemma 4 in \citet{chen2022learning}]
\label{lem:chen2022 lemma 4}
    For any $\pi \in \PiSR$ and reward function $r \in \bbR^{|\calS| \times|\calA|}$ such that $J_r^{\pi}(s) = J_r^\pi$ for all $s\in\calS$, we have $|V_{r,h}^{\pi}(s) - (H-h+1)J_r^\pi| \leq \spn(v_r^\pi)$ for $(s,h)\in \calS \times [H]$.
\end{lemma}

\subsection{Challenges and techniques: average-reward linear CMDPs}
Under the finite-horizon approximation framework, the main challenge in the linear setting can be summarized as follows: obtaining a refined regret bound in the finite-horizon regime remains nontrivial even for time-homogeneous $P$. More specifically, since directly estimating $P$ is intractable in the linear setting, one instead estimates the optimal value functions $\{V_h^*\}_{h=1}^H$. However, estimating this sequence introduces an additional multiplicative factor in $H$, since $\{V_h^*\}_{h=1}^H$ varies with $h$, even when $P$ is time-homogeneous~\citep{zihan2024horizonfree}. As $H$ may depend on $T$, this leads to an unfavorable regret bound in the linear setting~\citep{ghosh2023achieving}, highlighting a key challenge compared to the tabular setting~\citep{chen2022learning}. Hence, this necessitates a different approach tailored to this setting.

To improve the regret bound, our technique---inspired by \citet{hong2024reinforcement}---focuses on the boundedness condition of $V_h^*$, while still estimating the sequence $\{V_h^*\}_{h=1}^H$. In particular, by leveraging \Cref{lem:chen2022 lemma 4}, we have $\spn(V_h^*) \leq 2\zeta$ for some $\zeta$ that is independent of $H$, whereas \citet{ghosh2023achieving} relies on the trivial bound $\|V_h^*\|_\infty \leq H$. Here, $\zeta$ is a quantity related to the span of bias functions, which will be specified later in \Cref{subsec:alg}. This suggests replacing the $\calO(H)$ dependence in the regret bound with $\calO(\zeta)$, yielding a significant improvement, since $H$ may depend on $T$, whereas $\zeta$ does not. Consequently, our refined regret bound arises from exploiting the bound $\spn(V_h^*) \leq 2\zeta$, while still estimating the optimal value functions $\{V_h^*\}_{h=1}^H$.

To incorporate the described technique into the constrained setting, a natural approach is to consider a primal–dual extension of the algorithm for the unconstrained setting~\citep{hong2024reinforcement}. However, analyzing the resulting algorithm becomes nontrivial, as it requires additional tools. Importantly, our strong duality result plays a key role. To be more specific, our analysis based on \Cref{thm:strong duality} provides a new tool by establishing a lower bound on $\Regret(T)$. This, in turn, allows us to derive separate bounds on $\Regret(T)$ and $\Violation(T)$ from the composite Lagrangian regret $\Regret(T) + \lambda\Violation(T)$. In Section~\ref{subsec:analysis}, we provide additional details.

\subsection{Proposed algorithm}\label{subsec:alg}

\begin{algorithm}[t]
\caption{PD-LSCVI-UCB}
\label{alg:main}
\textbf{Input: } horizon $H=T^{1/3}$; $K= T^{2/3}$; clipping parameter $\zeta$; bonus factor $\beta$; Slater constant $\slater$; step size $\eta$;\\
\textbf{Initialize: } $\bmSigma_{1} = I$; $\lambda_1 =0$; $V_{H+1}^k(\cdot) \leftarrow 0$ for all $k\in [K]$;
\begin{algorithmic}[1]
\State Receive state $s_1$.
\For{$k=1,\ldots,K$}
\For{$h=H,\ldots,1$} 
    \State $\bfw_{h+1}^k \leftarrow \bmSigma_k^{-1} \sum_{\tau=1}^{k-1}\sum_{j=1}^H \phi(s_j^\tau,a_j^\tau) (V_{h+1}^k(s_{j+1}^\tau) - \min_{s'}V_{h+1}^k(s'))$ \label{line:PV}
    \State $Q_h^k(\cdot,\cdot) \leftarrow \left(r(\cdot,\cdot)+ \lambda_k g(\cdot,\cdot) + \phi(\cdot,\cdot)^\top \bfw_{h+1}^k +\min_{s'}V_{h+1}^k(s') + \beta\|\phi(\cdot,\cdot)\|_{\bmSigma_k^{-1}}\right)\wedge \left(1+\frac{2}{\slater}\right)H$ \label{line:Q}
    \State $\widetilde V_h^k(\cdot) \leftarrow \max_{a} Q_h^k(\cdot,a)$ \label{line:tilde V}
    \State $V_h^k(\cdot) \leftarrow \widetilde V_h^k(\cdot) \wedge (\min_{s'} \widetilde V_h^k(s') + 2\zeta)$ \label{line:V}
\EndFor
\For{$h=1,\ldots, H$}
    \State Take $a_h^k \in \argmax_{a} Q_h^k(s_h^k,a)$ and observe $s_{h+1}^k \sim P(\cdot|s_h^k,a_h^k)$
\EndFor
\State $s_1^{k+1} \leftarrow  s_{H+1}^k$
\State $\bmSigma_{k+1} \leftarrow \bmSigma_k + \sum_{h=1}^H \phi(s_h^k,a_h^k)\phi(s_h^k,a_h^k)^\top$
\State $\lambda_{k+1} \leftarrow \left[\lambda_k  + \eta\left(Hb -\sum_{h=1}^Hg(s_h^k,a_h^k)\right)\right]_{\left[0, \frac{2}{\slater}\right]}$
\EndFor
\end{algorithmic}
\end{algorithm}
We present \emph{primal–dual least-squares clipped value iteration with upper confidence bound} (PD-LSCVI-UCB, \Cref{alg:main}). The algorithm adopts the finite-horizon approximation framework with $K$ episodes and the horizon $H$. For brevity, let $s_h^k = s_{(k-1)H + h}$, $a_h^k = a_{(k-1)H+h}$, and let $H=T^{1/3}, \ K=T^{2/3}$ be an integer. In each episode $k\in [K]$, $Q_h^k, \widetilde V_h^k$ are computed via backward induction from $h=H$ to $1$, providing optimistic estimates under a composite reward function $r + \lambda_k g$, where $\lambda_k$ denotes the dual variable in episode $k$. The clipped value function estimate $V_h^k$ is then computed. The dual variable is updated based on the cumulative constraint violation in the episode, i.e., $Hb - \sum_{h=1}^H g(s_h^k,a_h^k)$, and clipped to the interval $[0, 2/\slater]$. Here, the clipping parameter $2/\slater$ is chosen because it is a strict upper bound on the optimal dual variable, as shown in the following lemma.
\begin{lemma}\label{lem:optimal dual variable bound}
    Suppose that \Cref{assum:Slater condition} holds. Then $\lambda^* \leq 1/\slater$.
\end{lemma}

\paragraph{Value function clipping (Line \ref{line:V})} 
This step is crucial for exploiting the boundedness of the span of optimal value functions. The key intuition is that if the optimal value functions have span bounded by $2\zeta$ for some $\zeta$, then the value function estimates can also be restricted to this range. Hence, even when the estimates are clipped at $2\zeta$, they continue to overestimate the optimal value functions, ensuring sufficient yet tight exploration.

To successfully implement value function clipping, we assume the existence of an optimal stationary policy with constant gain. This assumption is required to apply \Cref{lem:chen2022 lemma 4}, and has also been adopted in prior works~\citep{chen2022learning, ghosh2023achieving}. On the other hand, weakly communicating average-reward CMDPs need not admit an optimal policy; a detailed discussion is provided in \Cref{appendix:additional discussion}. This highlights a key difference from the unconstrained setting, where an optimal stationary policy with constant gain always exists. For clarity, we formalize this via the following assumption.
\begin{assumption}\label{assum:zeta}
An optimal policy $\pi^*$ for \eqref{eq:average-reward CMDP} exists. Its gains $J_r^*(s)$ and $J_g^*(s)$ are constant across $s \in \calS$. Moreover, the spans of its state bias functions, $\spn(v_r^*)$ and $\spn(v_g^*)$, are known.
\end{assumption}


Building on this assumption, we define the clipping parameter $\zeta$ as
\begin{equation}\label{eq:zeta}
    \zeta \triangleq \spn(v_r^*) + \frac{2}{\slater}\spn(v_g^*).
\end{equation}
Combining this with \Cref{lem:chen2022 lemma 4}, we obtain
\begin{equation*}
    \spn(V_{r,h}^*) \leq \max_{s,s'} \bigl|V_{r,h}^*(s) - (H-h+1)J_r^*\bigr|
    + \bigl|V_{r,h}^*(s') - (H-h+1)J_r^*\bigr|
    \leq 2\spn(v_r^*).
\end{equation*}
The same argument applies to $g$. As a consequence, since $\lambda_k \in [0, 2/\gamma]$, we have
\begin{equation}\label{eq:sp leq 2zeta}
    \spn\!\left(V_{r,h}^* + \lambda_k V_{g,h}^*\right)
    \leq \spn(V_{r,h}^*) + \lambda_k \spn(V_{g,h}^*)
    \leq 2\zeta.
\end{equation}
As required, $2\zeta$ serves as an upper bound on the span of the optimal value functions, motivating us to clip $\widetilde V_h^k$ at this level. Finally, PD-LSCVI-UCB achieves the following results.
\begin{theorem}\label{thm:regret analysis}
    Let $\calM = (\calS, \calA, P, r,g,b)$ be a weakly communicating linear CMDP with $|\calS|, |\calA| < \infty$, and let $s_1 \in \calS$ be any initial state. Suppose that Assumptions \ref{assum:Slater condition}, \ref{assum:zeta} hold. Let $K = T^{2/3}, H = T^{1/3}, \beta = \bigO((\zeta+1)d), \eta = 2/(\slater\sqrt{KH^2})$, and let $\spn(v_{\lambda^*}^*)$ be defined in \Cref{lem:regret lower bound}. With probability at least $1-3\delta$, PD-LSCVI-UCB (\Cref{alg:main}) guarantees
    \begin{align*}
        &\Regret(T) = \bigO\left(\left(\zeta + \frac{2}{\gamma}\right) T^{2/3} + (\zeta+1)(\sqrt{d^3 T} + d^2T^{1/3})\right),\\
        &\Violation(T) = \bigO\left(\left(\slater\zeta + 2\right) T^{2/3} + \slater(\zeta+1)(\sqrt{d^3 T} + d^2T^{1/3}) + \slater\spn({v_{\lambda^*}^*})\sqrt{T}\right),
    \end{align*}
    where $\bigO(\cdot)$ hides polynomial factors in $\log(dT|\calA|\zeta/(\delta\gamma))$.
\end{theorem}

\paragraph{Comparisons with prior works}
We provide a comparison with prior algorithms for learning average-reward linear MDPs. \citet{hong2024reinforcement} proposed a clipped value iteration algorithm for learning average-reward unconstrained linear MDPs. Beyond the primal–dual extension, there is another key difference. Their work employs a discounted approximation framework, whereas ours adopts a finite-horizon approximation. Under the discounted approximation, the horizon can increase up to $H=\calO(T)$, as it is determined by the doubling trick (see Line 5 of their algorithm). However, this becomes problematic in the constrained setting, since $\lambda_{k+1}-\lambda_k$ can also grow up to $\calO(\eta T)$, which makes the dual variable unstable. To overcome this, we adopt a finite-horizon approximation with $H = T^{1/3}$ and show that their technique extends to this framework.

\citet{ghosh2023achieving} proposed a primal–dual algorithm for learning average-reward linear CMDPs. While their algorithm is also based on the finite-horizon approximation, it does not employ value function clipping (Line~\ref{line:V}), and thus the tighter span bound~\eqref{eq:sp leq 2zeta} is not fully leveraged. Moreover, their analysis relies on strong duality for finite-horizon CMDPs, resulting in regret bounds of $\bigO(T^{3/4})$. In contrast, our method employs value function clipping, and its analysis is enabled by strong duality for average-reward CMDPs (\Cref{thm:strong duality}), yielding improved regret bounds of $\bigO(T^{2/3})$. Indeed, these bounds match those of the tabular setting in terms of $T$~\citep{chen2022learning}.

\subsection{Outline of analysis}\label{subsec:analysis}
The overall strategy for the analysis consists of the following two steps. In the first step, we derive an upper bound $\Delta$ on the composite Lagrangian regret $\Regret(T) + \lambda \Violation(T)$ that holds for any $\lambda \in [0, 2/\slater]$. This immediately yields an upper bound on $\Regret(T)$ by taking $\lambda = 0$. On the other hand, bounding $\Violation(T)$ requires a lower bound on $\Regret(T)$, since $\Violation(T) \leq \frac{\slater}{2}(\Delta - \Regret(T))$. Therefore, in the second step, we obtain such a lower bound by leveraging strong duality.

\paragraph{Step 1: Upper bound on $\bm{\Regret(T) + \lambda \Violation(T)}$}
To upper bound the composite Lagrangian regret, we first provide the following useful lemmas. In particular, \Cref{lem:good event} captures the estimation error and characterizes the optimistic bonus. Then, \Cref{lem:optimism} shows that the clipped value function estimates optimistically estimate the composite value functions.
\begin{lemma}\label{lem:good event} 
Let $\beta = \bigO((\zeta+1)d)$. With probability at least $1-\delta$, for $(s,a,h,k) \in \calS\times\calA \times [H] \times [K]$,
\begin{align*}
    |\phi(s,a)^\top(\bfw_{h+1}^{k} - \bfw_{h+1}^{k,*})| \leq  \beta\|\phi(s,a)\|_{\bmSigma_k^{-1}}.
\end{align*}    
where $\bfw_{h+1}^{k,*} = \sum_{s} \psi(s)(V_{h+1}^k(s) - \min_{s'}V_{h+1}^k(s'))$.
\end{lemma}
\begin{lemma}\label{lem:optimism}
    Suppose that \Cref{lem:good event} holds. Then for all $(s,h,k) \in \calS \times[H] \times [K]$,
    \begin{align*}
     V_{r,h}^{*}(s) + \lambda_k  V_{g,h}^{*}(s) \leq V_{h}^k(s).
\end{align*}
\end{lemma}
Given these lemmas, we can decompose $\Regret(T) + \lambda \Violation(T)$ as follows.
\begin{align}\label{eq:decomp}
\begin{aligned}
    \Regret(T) + \lambda \Violation(T)
    &\leq K\zeta + \sum_{k=1}^K\sum_{h=1}^H 2\beta\|\phi(s_h^k,a_h^k)\|_{\bmSigma_k^{-1}}\\
    &\quad+ \sum_{k=1}^K\sum_{h=1}^H \left( P V_{h+1}^k(s_h^k,a_h^k) - V_{h+1}^k(s_{h+1}^k)\right)+\frac{\lambda^2}{2\eta} + \frac{\eta KH^2}{2}.
\end{aligned}
\end{align}
Here, the second term can be bounded by the elliptical potential lemma, and the third term is bounded, as it is a sum of martingale differences. Finally, with the parameter choice given in \Cref{thm:regret analysis}, we derive the desired bound, as stated in the following lemma.
\begin{lemma}\label{lem:Regret + lambda Volation}
    Suppose that \Cref{lem:good event} holds. With probability at least $1-\delta$, for any $\lambda \in [0, 2/\slater]$, 
    \[
        \Regret(T) + \lambda \Violation(T) = \bigO\left(\left(\zeta + \frac{2}{\gamma}\right) T^{2/3} + (\zeta+1)(\sqrt{d^3 T} + d^2T^{1/3})\right).
    \]
\end{lemma}

\paragraph{Step 2: Lower bound on $\bm{\Regret(T)}$} 
We achieve this by leveraging \Cref{thm:strong duality}, together with the Bellman optimality equation for weakly communicating average-reward unconstrained MDPs~\citep{puterman1994markov}. First, we consider a bounded optimal dual variable $\lambda^*$, which is guaranteed by \Cref{thm:strong duality} and \Cref{lem:optimal dual variable bound}. For a reward $r + \lambda^*(g-b)$, the Bellman optimality condition for weakly communicating unconstrained MDPs implies that there exists $J_{\lambda^*}^*(s_1) \in \bbR$, $v_{\lambda^*}^*: \calS\to\bbR$ such that for any $(s,a) \in\calS\times \calA$,
\begin{equation}\label{eq:bellman optimality equation}
    J_{\lambda^*}^*(s_1) - r(s,a) \geq -\lambda^*(b-g(s,a)) + Pv_{\lambda^*}^*(s,a) - v_{\lambda^*}^*(s).
\end{equation}
Given $\spn(v_{\lambda^*}^*)$, we are now ready to state a lemma about a lower bound on $\Regret(T)$.
\begin{lemma} \label{lem:regret lower bound}
Let $v_{\lambda^*}^*$ be defined in \eqref{eq:bellman optimality equation}. With probability at least $1-\delta$,
\begin{align*}
    \Regret(T) \geq -\lambda^* \Violation(T) - \spn({v_{\lambda^*}^*})\sqrt{2T\log(1/\delta)} - \spn({v_{\lambda^*}^*}).
\end{align*}    
\end{lemma}
The proof proceeds as follows. When the state and action spaces are finite, $J_{r+\lambda^* (g-b)}^\pi(s_1)$ can be decomposed as $J_r^\pi(s_1) + \lambda^* (J_g^\pi(s_1) - b)$. Therefore, maximizing $J_{r+\lambda^*(g-b)}^\pi(s_1)$ over $\pi \in \PiSR$ coincides with $D(\lambda^*)$. By strong duality, we deduce that $J_r^*(s_1) = D(\lambda^*) =  J_{\lambda^*}^*(s_1)$.

Now, recall the Bellman optimality condition~\eqref{eq:bellman optimality equation}. Since we showed that $J_r^*(s_1) = J_{\lambda^*}^*(s_1)$ by strong duality, summing the left-hand side along the trajectory recovers $\Regret(T)$, while the right-hand side provides its lower bound. To formalize this, given the trajectory $\{(s_t,a_t)\}_{t=1}^T$, we substitute $s=s_t, a=a_t$, and sum over $t=1$ to $T$. After further computation, we obtain the desired lower bound on $\Regret(T)$.


\section{Conclusion}\label{sec:conclusion}
This paper studied weakly communicating average-reward CMDPs and established strong duality under this general setting, despite the absence of an LP formulation and the resulting nonconvexity. The key insight is that the geometric structure of the occupation measure set is sufficient to recover strong duality. Building on this, we proposed a primal–dual algorithm for linear CMDPs that achieves $\bigO(T^{2/3})$ regret and constraint violation bounds by combining clipped value iteration with a strong duality–based analysis to control the composite Lagrangian regret and its decomposition. Despite these results, several important directions remain open. First, it is of interest to extend the strong duality result to infinite state and action spaces, as the current approach relies on finite-dimensional LP arguments. Second, it remains open whether the regret and constraint violation bounds can be further improved to $\bigO(\sqrt{T})$, which remains unresolved for both tabular and linear average-reward CMDPs under the weakly communicating assumption.

\bibliographystyle{plainnat}
\bibliography{ref}

\newpage

\appendix

\newpage

\section{Related work}

\paragraph{Average-reward CMDPs}
Learning in average-reward CMDPs has been studied under several algorithmic frameworks and structural assumptions. Early works considered tabular CMDPs with long-term average constraints and established regret and constraint-violation guarantees~\citep{agarwal2021markov, agarwal2022regret, wei2022provably, singh2022learning, chen2022learning}. In particular, \citet{chen2022learning} provided algorithms under both ergodic and weakly communicating assumptions, while \citet{ghosh2023achieving} extended this line to linear function approximation. More recently, \citet{pmlr-v235-provodin24a} studied posterior-sampling-based exploration for average-reward CMDPs, and \citet{wei2026nearoptimal} investigated near-optimal sample complexity bounds under a generative model. Related constrained formulations have also been studied through concave utility objectives and zero-violation guarantees~\citep{agarwal2022concave}. Another line of work studies average-reward CMDPs through policy-gradient and actor--critic methods. \citet{bai2024learning} considered general policy parameterization and developed a primal--dual policy-gradient algorithm. Subsequent works further studied convergence, regret, and actor--critic variants under different parameterization and structural assumptions~\citep{ xu2025global, satheesh2025primal, satheesh2026regret, satheesh2026global}. These works provide important progress for average-reward CMDPs, but they typically rely on stronger assumptions such as ergodicity or unichain structure, or focus on different learning models. In contrast, our work establishes strong duality under the weakly communicating assumption and uses it to obtain improved regret and constraint violation bounds for average-reward linear CMDPs.

\paragraph{Average-reward unconstrained MDPs}
There is also a large literature on learning unconstrained average-reward MDPs. In the tabular setting, \citet{jaksch2010near} proposed UCRL2 for communicating MDPs, while posterior-sampling approaches were studied in communicating and weakly communicating settings~\citep{agrawal2017optimistic, ouyang2017learning}. For weakly communicating MDPs, \citet{bartlett2012regal} introduced REGAL using bias-span regularization, and \citet{fruit2018efficient} developed a computationally efficient bias-span-constrained algorithm. Model-free algorithms for average-reward MDPs were studied by \citet{wei2020model, zhang2023sharper}, and recent works further investigate planning and sample-complexity aspects of average-reward MDPs~\citep{agrawal2024optimistic, zurek2024span, lee2023accelerating, lee2025near, lee2025optimal}. For linear function approximation, \citet{wei2021learning} proposed algorithms for infinite-horizon average-reward linear MDPs, including both statistically optimal and computationally efficient variants under different assumptions. More recently, \citet{hong2024reinforcement} studied average-reward linear MDPs via approximation by discounted-reward MDPs and introduced value function clipping to obtain sharp regret guarantees, while \citet{hong2025computationally} further improved computational efficiency by avoiding dependence on the state-space size in the clipping step. Policy-gradient methods have also been studied for unconstrained average-reward MDPs~\citep{bai2024regret, ganesh2024order, ganesh2025regret}, as well as extensions beyond the unichain setting~\citep{lee2026policy}. Our algorithm is inspired by the value-function clipping technique developed for unconstrained average-reward linear MDPs~\citep{hong2024reinforcement}, but extending this technique to CMDPs requires a primal--dual formulation and a strong-duality-based analysis.

\paragraph{Finite-horizon CMDPs}
Compared with the average-reward setting, finite-horizon CMDPs have been studied under a broader variety of models, feedback structures, and safety requirements. In the tabular setting, \citet{efroni2020exploration} studied the exploration--exploitation tradeoff in CMDPs through both LP-based and primal--dual approaches. Subsequent works refined the guarantees by considering zero or bounded constraint violation~\citep{liu2021learning, bura2022dope, yu2025improved}, and optimal strong regret and violation~\citep{muller2024truly,stradi2025optimal}. Another line of work studies online CMDPs with stochastic or adversarial losses and constraints~\citep{qiu2020upper, stradi2024online, stradi2025policy, stradi2025learning, zhu2025an}. Under function approximation, \citet{ding2021provably} and \citet{ghosh2022provably} developed provably efficient algorithms for finite-horizon linear CMDPs, and subsequent works further studied episode-wise safety, instantaneous constraints, non-stationarity, and adversarial losses in linear or related CMDP models~\citep{amani2021safe, wei2023provably, wei2024safe, kitamura2025provably, roknilamouki2025provably, yu2026primaldual, yu2026near}. These works focus on finite-horizon or episodic CMDPs, whereas our work studies infinite-horizon average-reward CMDPs under the weakly communicating assumption.



\section{Additional discussions}\label{appendix:additional discussion}

\paragraph{Absence of optimal stationary policies} 
We explain why \eqref{eq:average-reward CMDP} may fail to admit an optimal stationary policy under the weakly communicating assumption for a fixed initial state $s_1$. Equivalently, it suffices to show that \eqref{eq:occ average-reward CMDP} may have a finite supremum that is not attained. The main reason is that $\calPSR{s_1}$ may miss some boundary points (\Cref{lem:relint & ext}). As a result, the feasible set $\calPSR{s_1}\cap \{q: g^\top q \geq b\}$ may exclude some extreme points for certain $g \in \bbR^{|\calS|\times|\calA|}$ and $b \in \bbR$. Consequently, for such $g$ and $b$, the optimization problem may fail to attain its optimum for some $r \in \bbR^{|\calS| \times|\calA|}$. This phenomenon is illustrated in \Cref{fig:no optimal}.

\begin{figure}[H]
\centering
\includegraphics[width=0.3\linewidth]{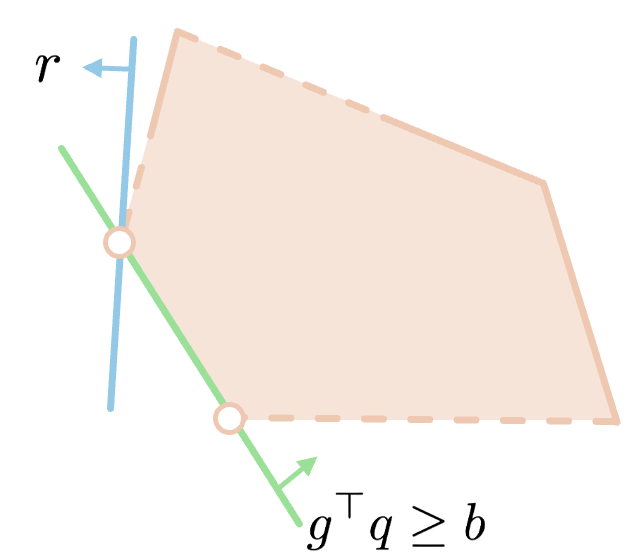}
\caption{{Illustration of an average-reward CMDP over the occupation measure space.} The green line represents the constraint $g^\top q \geq b$, the red region represents the feasible region $\calPSR{s_1}\cap\{q: g^\top q \geq b\}$, and the blue arrow indicates the direction of $r$.}
\label{fig:no optimal}
\end{figure}
We provide a more concrete example, inspired by Example 3.1 in \citet{mannor2005empirical}. Let $\calS = \{s_1, s_2\}$, $\calA = \{a_1, a_2\}$, $b = 1/2$, and let $r(s_2, a_2) = 1$, $g(s_1,a_1) = 1$, otherwise $0$. The transition kernel is deterministic, which is illustrated as follows.
\begin{figure}[H]
    \centering
    \begin{tikzpicture}[
    >=Stealth,
    node distance=4cm,
    state/.style={circle, draw, minimum size=1.1cm, thick},
    every edge/.style={draw, thick, ->},
    action/.style={midway, above, font=\small}
    ]
    
    \node[state] (s1) {$s_1$};
    \node[state, right of=s1] (s2) {$s_2$};
    
    \path (s1) edge[loop left] node[action] {$a_1$} (s1);
    \path (s2) edge[loop right] node[action] {$a_2$} (s2);
    
    \path (s1) edge[bend left=18] node[action] {$a_2$} (s2);
    \path (s2) edge[bend left=18] node[action] {$a_1$} (s1);
    \end{tikzpicture}
\end{figure}
Note that this MDP is weakly communicating because $s_1,s_2$ become recurrent under a policy taking $a_2$ at $s_1$, and $a_1$ at $s_2$. Given these, we consider the following average-reward CMDPs:
\begin{equation}\label{eq:counterexample}
    \sup_{\pi \in \PiSR}\; J_r^\pi(s_1) \quad \text{s.t.} \quad J_g^\pi(s_1) \geq 1/2.
\end{equation}
To find an optimal stationary policy, consider $\pi(a_2 | s_1) = \alpha$ and $\pi(a_1|s_2) = \beta$ for some $\alpha, \beta \in [0,1]$. Suppose that $\alpha, \beta > 0$. In this case, the resulting Markov chain is
\[
    P_\pi = \begin{bmatrix}
        1-\alpha & \alpha\\
        \beta & 1-\beta
    \end{bmatrix}.
\]
Since it is an irreducible chain, its (unique) stationary distribution satisfies $d_\pi(s_1) = \beta / (\alpha + \beta)$ and $d_\pi(s_2) = \alpha / (\alpha + \beta)$. By Proposition 8.9.1 in \citet{puterman1994markov}, we can deduce that $\occ{\pi}{s_1}(s,a) = d_\pi(s) \pi(a|s)$ for $(s,a) \in \calS \times \calA$. It follows that 
\begin{align*}
    \occ{\pi}{s_1}(s_1,a_1) = \frac{\beta(1-\alpha)}{\alpha+\beta}, \; \occ{\pi}{s_1}(s_2,a_2) = \frac{\alpha(1-\beta)}{\alpha+\beta}. 
\end{align*}
To achieve the supremum, for $\varepsilon \in (0,1/3)$, take 
\[
    \alpha = \varepsilon, \quad \beta = \frac{\varepsilon}{1-2\varepsilon}.
\]
Then, $\pi$ is feasible, as $g^\top \occ{\pi}{s_1} = 1/2$ for any $\varepsilon \in (0,1/3)$, and
\[
    r^\top \occ{\pi}{s_1} = \frac{1}{2} - \frac{\varepsilon}{1-\varepsilon}.
\]
As $\varepsilon \to 0+$, the supremum of \eqref{eq:counterexample} becomes $1/2$. However, there is no fixed $\alpha,\beta >0$ that achieves $r^\top\occ{\pi}{s_1} =1/2$ and $g^\top \occ{\pi}{s_1} \geq 1/2$, because $\occ{\pi}{s_1}(s_1,a_1) + \occ{\pi}{s_1}(s_2,a_2) < 1$ for any $\alpha,\beta >0$. Moreover, this can be easily shown when $\alpha = 0$ or $\beta =0$. Consequently, this example demonstrates that weakly communicating average-reward CMDPs can have no optimal stationary solutions.

\paragraph{Importance of strong duality over stationary policies}
We explain why strong duality over stationary policies is particularly useful for learning problems. Specifically, one may ask why we do not instead consider \eqref{eq:average-reward CMDP} over the history-dependent randomized policy class, which is a more general formulation:
\begin{align}\label{eq:HR average-reward CMDP}
    J_r^{*,\textnormal{HR}}(s_1)
    \quad\triangleq\quad
    \max_{\pi \in \Pi^{\textnormal{HR}}} \ J_r^\pi(s_1)
    \; \textnormal{s.t.}\;
    J_g^\pi(s_1) \geq b,
    \tag{HR-CMDP}
\end{align}
where $\Pi^{\textnormal{HR}}$ denotes the set of history-dependent policies whose occupation measures exist. One may then analyze regret and constraint violation defined with respect to $J_r^{*,\textnormal{HR}}(s_1)$ using strong duality over $\Pi^{\textnormal{HR}}$. While this approach is possible, it prevents us from leveraging favorable properties specific to stationary policies, such as \Cref{lem:chen2022 lemma 4}, thereby hindering tighter regret bounds. To resolve this issue, we must answer the following questions: (i) can we restrict the policy class to the stationary policy class~$\Pi$ without loss of optimality, and (ii) does strong duality hold over $\Pi$?

Our results in \Cref{sec:strong duality} answer these questions. In particular, we can show that $J_r^{*, \textnormal{HR}}(s_1) = J_r^{*}(s_1)$, where $J_r^{*}(s_1)$ denotes the supremum of \eqref{eq:average-reward CMDP} over $\Pi$. This follows from \Cref{lem:relint & ext} together with the fact that the occupation measure set induced by $\Pi^{\textnormal{HR}}$ coincides with $\calP$ (Theorem 3.1 in \citet{mannor2005empirical}). Consequently, this justifies restricting the policy class from $\Pi^{\textnormal{HR}}$ to $\Pi$ without loss of optimality, which in turn allows us to leverage favorable properties of stationary policies. Finally, \Cref{thm:strong duality} fills the remaining gap by establishing strong duality over $\Pi$.

\paragraph{Computational complexity}
We provide the computational complexity of \Cref{alg:main}. Since our algorithm is a primal--dual extension of Algorithm 2 in \citet{hong2024reinforcement}, its complexity is comparable to that of the original algorithm. The main difference arises from our use of the finite-horizon approximation. Specifically, in each episode $k \in [K]$, the algorithm performs $\calO(H)$ value iterations. In each step $h \in [H]$, it computes $\min_{s'} \widetilde V_h^k(s')$, which requires $\calO(d^2|\calS||\calA|)$ operations. In addition, other operations require $\calO(d^2|\calA|T)$ operations. Therefore, each episode requires $\calO(H d^2 |\calS||\calA| T)$ operations, and the total computational complexity is $\calO(d^2 K H |\calS||\calA| T)$. Since $HK = T$, this simplifies to $\calO(d^2 |\calS||\calA| T^2)$.

\paragraph{Limitations of Algorithm~\ref{alg:main}}
We discuss the limitations of \Cref{alg:main}. First, its computational complexity depends on $|\calS|$. In particular, the value function clipping step computes $\min_{s'} \widetilde V_h^k(s')$, requiring at least $|\calS|$ operations, which can be prohibitive when the state space is large. Second, its regret bound is suboptimal, as it does not achieve $\bigO(\sqrt{T})$, whereas optimal regret rates are attainable in the unconstrained setting~\citep{hong2024reinforcement, hong2025computationally}. In our case, the suboptimal $\bigO(T^{2/3})$ regret arises from controlling the dual variable. More precisely, in the regret decomposition~\eqref{eq:decomp}, the relevant terms are
\[
    \frac{\lambda^2}{2\eta} + \frac{\eta KH^2}{2}.
\]
Under the finite-horizon approximation framework, we have $KH = T$, and thus these terms can be written as $\frac{\lambda^2}{2\eta} + \frac{\eta H T}{2}$. Since $H$ depends on $T$, obtaining a $\calO(\sqrt{T})$ bound for these terms is nontrivial.

For the constrained setting, we illustrate the challenges in designing an optimal algorithm, while its computational complexity is independent of $|\calS|$. A natural approach is to extend algorithms designed for the unconstrained setting to the constrained setting. Indeed, in the unconstrained setting, \citet{hong2025computationally} proposed an optimal algorithm. In particular, their algorithm proceeds as follows. To avoid $\calO(|\calS|)$ computational complexity, the algorithm clips the value function estimates using only the visited states, replacing $\min_{s'} \widetilde V_h^k(s')$ with $\min_{s \in \calS_{\mathrm{visit}}} \widetilde V_h^k(s)$, where $\calS_{\mathrm{visit}}$ denotes the set of visited states. While this modification reduces the computational cost, it introduces additional errors. To control this error, it applies a deviation control technique to ensure that $\|V_{h+1}^k - V_{h+1}^{k+1}\|_\infty$ remains small.

However, extending this approach to the constrained setting introduces additional challenges, primarily due to the presence of the dual variable. In particular, when adapting their algorithm to a primal--dual framework, we may introduce a dual variable $\lambda_k$, for example through the modified reward $r + \lambda_k g$. Since $\lambda_k$ varies across iterations $k$, controlling the deviation term $\|V_{h+1}^k - V_{h+1}^{k+1}\|_\infty$ becomes significantly more difficult. Therefore, designing an optimal algorithm with $|\calS|$-independent computational complexity in our setting remains an open problem.

\section{Additional preliminaries}\label{appendix:additional preliminary}
We introduce some additional notions that are useful for proving strong duality and regret analysis. For more details, we refer the reader to \citet{puterman1994markov, norris1998markov}.

\paragraph{Definitions for Markov chains} Consider a Markov chain $P_\pi$ on a finite state $\calS$. Let $P_\pi^{(n)}(s'|s)$ denote $\bbP(s_{t+n} = s'|s_t = s)$. We say $s'$ is \emph{accessible} from $s$ if $P_\pi^{(n)}(s'|s) > 0$ for some $n\geq 0$, denoted by $s \to s'$, and $s$ \emph{communicates} with $s'$ if $s \to s',\ s' \to s$, denoted by $s \leftrightarrow s'$.  

A set $C \subseteq \calS$ is \emph{closed} if there is no state in $\calS \backslash C$ that is accessible from any state in $C$. A closed set $C$ is \emph{irreducible} if there is no proper subset that is closed. Let $\tau_s$ denote the time of first visit (return to $s$ starting from $s$). A state $s$ is \emph{recurrent} if $\bbP(\tau_s <\infty) = 1$, and \emph{transient} if $\bbP(\tau_s < \infty) < 1$. We say $d_\pi:\calS \to \bbR_+$ is a \emph{stationary distribution} for $P_\pi$ if $\sum_{s}d_\pi(s) = 1$ and $\sum_{s'}P_\pi(s|s')d_\pi(s') = d_\pi(s)$ for all $s$. Equivalently, we can write it as $\bm{1}^\top d_\pi = 1$ and $P_\pi^\top d_\pi = d_\pi$ using matrix notation.

Since $\leftrightarrow$ is an equivalence relation, $\calS$ can be partitioned into \emph{communicating classes}. Moreover, for each communicating class $C$, all states in $C$ are either recurrent or transient. A communicating class $C$ is \emph{recurrent class} if all states in $C$ are recurrent, and \emph{transient class} if all states in $C$ are transient. Let $\gcd\{n >0: P_\pi^{(n)}(s|s) > 0\}$ denote the period of state $s$, and all states in a recurrent class share the same period. We call a class \emph{periodic} if the period of states exceeds $1$, and \emph{aperiodic} if $1$.

A Markov chain $P_\pi$ is \emph{irreducible} if it consists of a single closed class, \emph{unichain} if it consists of a single set of closed irreducible states and a set of (possibly empty) transient states, and otherwise \emph{multichain}. 

\paragraph{MDP classifications}
An MDP $\calM$ is called \emph{ergodic} if, for every stationary policy, the induced transition matrix is irreducible and aperiodic.   An MDP $\calM$ is called \emph{unichain} if, for every deterministic stationary policy, the induced transition matrix is unichain. An MDP $\calM$ is called \emph{(general) multichain} if there exists at least one stationary policy whose induced transition matrix contains one or more closed irreducible recurrent classes. These classes satisfy {ergodic}$ \subseteq$  {unichain} $\subseteq$ {weakly communicating} $\subseteq$ {multichain}.


Next, from Lemmas \ref{lem:prop 8.9.1} to \ref{lem: J_r = sup J_r}, we provide standard results for average-reward MDPs with finite $\calS, \calA$.

\begin{lemma}[Proposition 8.9.1 in \citet{puterman1994markov}]\label{lem:prop 8.9.1}
    Let $\calM = (\calS, \calA, P, r)$ be an MDP with $|\calS|, |\calA| < \infty$, and let $s_1\in\calS$ be any initial state. For any $\pi \in \PiSR$, the occupation measure $\occ{\pi}{s_1}$ exists.
\end{lemma}

\begin{corollary}\label{lem:J -> inner product}
    Let $\calM = (\calS, \calA, P, r)$ be an MDP with $|\calS|, |\calA| < \infty$, and let $s_1\in\calS$ be any initial state. Let $r:\calS \times \calA \to \bbR$ be any bounded reward function. For any $\pi \in \PiSR$, we have $J_r^\pi(s_1) = r^\top \occ{\pi}{s_1}$.
\end{corollary}
\begin{proof}
    Note that
    \begin{align*}
        J_r^\pi(s_1) = \lim_{T\to\infty} \frac{1}{T}\bbE_{P,\pi}\left[\sum_{t=1}^T r(s_t,a_t) | s_1\right] = \lim_{T\to\infty} \frac{1}{T}\bbE_{P,\pi}\left[\sum_{t=1}^T \sum_{s,a} r(s,a) \mathds{1}\{s_t=s,a_t=a\} | s_1\right].
    \end{align*}
    By \Cref{lem:prop 8.9.1}, since $\occ{\pi}{s_1}(s,a) = \lim_{T\to\infty} \frac{1}{T}\bbE_{P,\pi}\left[\sum_{t=1}^T\mathds{1}\{s_t=s,a_t=a\} | s_1\right]$ exists and $\calS,\calA$ are finite, we have
    \begin{align*}
        J_r^\pi(s_1) = \sum_{s,a} r(s,a)\lim_{T\to\infty} \frac{1}{T}\bbE_{P,\pi}\left[\sum_{t=1}^T \mathds{1}\{s_t=s,a_t=a\} | s_1\right] = r^\top \occ{\pi}{s_1}.
    \end{align*}
\end{proof}

\begin{lemma}[Theorem 8.9.3 in \citet{puterman1994markov}]\label{lem:thm 8.9.3}
    Let $\calM = (\calS, \calA, P, r)$ denote an MDP with $|\calS| < \infty$ and $|\calA| < \infty$, and let $s_1\in\calS$ be any initial state. Then we have 
    \[
        \conv(\calPSD{s_1})=\clconv(\calPSD{s_1}) = \clconv(\calPSR{s_1}).
    \]
\end{lemma}
where $\conv(\cdot), \clconv(\cdot)$ denote the convex hull and closed convex hull, respectively, $\PiSD$ denotes the set of stationary deterministic policies, and $\calPSD{s_1} = \{\occ{\pi}{s_1}: \pi \in \PiSD\}$.
\begin{proof}
    Since $\clconv(\calPSD{s_1}) = \clconv(\calPSR{s_1})$ directly follows from Theorem 8.9.3 in \citet{puterman1994markov}, it is sufficient to show that $\conv(\calPSD{s_1})=\clconv(\calPSD{s_1})$. Since $|\calS|,|\calA| < \infty$, the number of stationary deterministic policies is at most $|\calA|^{|\calS|}$, which is finite. Moreover, $\calPSD{s_1}$ is finite, since it is defined as occupation measures with respect to stationary deterministic policies. Note that a closed convex hull of a finite set is equal to its convex hull. This concludes the proof.
\end{proof}


\begin{lemma}\label{lem: J_r = sup J_r}
    Let $\calM = (\calS, \calA, P, r)$ denote an MDP with $|\calS| < \infty$ and $|\calA| < \infty$, and let $s_1\in\calS$ be any initial state. Let $r:\calS \times \calA \to \bbR$ be any bounded reward function. There exists $\pi^* \in \PiSD$ such that $J_r^{\pi^*}(s_1) = \sup_{\pi \in \PiSR} J_r^\pi(s_1)$.
\end{lemma}
\begin{proof}
    Due to the definition of $\calPSR{s_1}$ and \Cref{lem:J -> inner product}, $\sup_{\pi \in \PiSR} J_r^\pi(s_1)$ can be rewritten as
    \begin{align*}
        \sup_{q \in \calPSR{s_1}} r^\top q.
    \end{align*}
    Since $\calPSR{s_1}\subseteq\clconv(\calPSR{s_1})$ and, by \Cref{lem:thm 8.9.3}, $\clconv(\calPSR{s_1})=\conv(\calPSD{s_1})$, it follows that
    \begin{align*}
        \sup_{q \in \calPSR{s_1}} r^\top q \leq \sup_{q \in \clconv(\calPSR{s_1})} r^\top q = \sup_{q \in \conv(\calPSD{s_1})} r^\top q.
    \end{align*}
    Note that $\calPSD{s_1}$ is finite. Then $\conv(\calPSD{s_1})$ is a non-empty polytope, and the right-hand side becomes an LP. Therefore, the maximizer is attained at an extreme point of $\conv(\calPSD{s_1})$---denoted by $q^*$---and it is clear that $q^* \in \calPSD{s_1}$, i.e., 
    \begin{align*}
        \exists q^* \in \calPSD{s_1}: \ r^\top q^* = \sup_{q \in \conv(\calPSD{s_1})} r^\top q.
    \end{align*}
    This implies that
    \begin{align*}
        \sup_{q \in \calPSR{s_1}} r^\top q \leq r^\top q^*.
    \end{align*}
    Moreover, note that $q^* \in \calPSD{s_1} \subseteq \calPSR{s_1}$. This implies that
    \begin{align*}
        r^\top q^* \leq \sup_{q \in \calPSR{s_1}} r^\top q.
    \end{align*}
    Finally, we show that there exists $q^* \in \calPSD{s_1}$ such that $r^\top q^* = \sup_{q \in \calPSR{s_1}} r^\top q$. By the definition of $\calPSD{s_1}$, there exists some stationary deterministic policy $\pi^*$ that induces $q^*$. By \Cref{lem:J -> inner product}, $J_r^{\pi^*}(s_1) = r^\top q^*$, and this completes the proof.
\end{proof}

\section{Missing proofs for Section~\ref{sec:strong duality}}\label{appendix:Strong duality}
In this section, we provide a rigorous proof for strong duality when the underlying MDP structure is weakly communicating. Recall that the strong duality result we are interested in is formulated as
\[
    \begin{array}{ccc}
         \begin{aligned}
             \sup_{\pi \in \PiSR} \quad& J_r^\pi(s_1) \\
             \textnormal{s.t.} \quad& J_g^\pi(s_1) \geq b
         \end{aligned}
         &=&
         \min_{\lambda\geq0}\sup_{\pi\in\PiSR}  J_r^\pi(s_1) + \lambda(J_g^\pi(s_1) - b).
    \end{array}
\]

To show this, we first introduce \Cref{lem:LP strong duality}, which states strong duality property for LPs, and Lemmas \ref{lem:relint equiv} - \ref{lem:relint}, which are useful for proving strong duality for average-reward CMDPs. Subsequently, in \Cref{subsec:proof of lem:relint & ext}, we provide a proof for \Cref{lem:relint & ext}, which characterizes the structure of $\calPSR{s_1}$ with arbitrarily chosen $s_1$. Finally, based on these results, we prove \Cref{thm:strong duality} in \Cref{subsec:proof of thm:strong duality}.

\begin{lemma}[LP Strong Duality]\label{lem:LP strong duality}
    Let $\calX \subseteq \bbR^n$ be a nonempty polytope, and let $r,g \in \bbR^n$ and $b\in\bbR$. Consider a primal LP $P^* = \max_{x\in\calX} r^\top x \ \textnormal{s.t.} \ g^\top x \geq b$. Define the corresponding dual functions as $D(\lambda) = \max_{x\in\calX} r^\top x + \lambda(g^\top x -b)$. If the primal LP is feasible, then there exists $\lambda^* \geq 0$ such that $\lambda^* \in \argmin_{\lambda \geq 0} D(\lambda)$, and moreover, $P^* = D(\lambda^*)$.
\end{lemma}
\begin{proof}
    This is a special case of Proposition 5.3.6 in \citet{bertsekas2009convex}.
\end{proof}

\begin{lemma}\label{lem:relint equiv}
    Let $\calX = \{x \in \bbR^n: x\geq 0, Ax = b\}$ be a non-empty polyhedron, and let $I_+ = \{i \in [n]: \exists x \in \calX \textnormal{ s.t. } x_i > 0\}$. Suppose that $I_+$ is non-empty. For $x \in \calX$, we have $x \in \relint(\calX)$ if and only if $x_i > 0$ for all $i \in I_+$.
\end{lemma}
\begin{proof}
    $(\Rightarrow)$
    Let $x\in\relint(\calX)$. Fix $i \in I_+$. By the definition of $I_+$, we can always find $y \in \calX$ such that $y_i >0$. By the definition of $\relint(\calX)$, there exists $\epsilon >1$ such that $\epsilon x + (1-\epsilon) y \in \calX$. By the definition of $\calX$, it follows that $\epsilon x_i + (1-\epsilon)y_i \geq 0$. Since $\epsilon >1$ and $y_i >0$, we have $x_i \geq \frac{(\epsilon-1)y_i}{\epsilon} > 0$. This implies that $x_i > 0$ for all $i\in I_+$.

    $(\Leftarrow)$
    Let $x \in \calX$ and $x_i > 0$ for all $i \in I_+$. Fix any $y \in \calX$. By the definition of $I_+$, we know that $y_j = 0$ for all $j \in I_+^c$. Take $\epsilon = 1 + \frac{\min_{i\in I_+} x_i}{\max\{1,\|y\|_\infty\}}$, so $\epsilon > 1$. We claim that $\epsilon x + (1-\epsilon)y \in \calX$. To show this, first it is clear that $A(\epsilon x + (1-\epsilon)y) = b$. Moreover, for $i\in I_+$, $\epsilon x_i + (1-\epsilon) y_i \geq 0$, due to our choice of $\epsilon$. For $j\in I_+^c$, $\epsilon x_j + (1-\epsilon) y_j = 0$. Therefore, by the definition of $\relint(\calX)$, $x \in \relint(\calX)$.
\end{proof}

\begin{lemma}\label{lem:MC}
    Let $P_\pi$ be a Markov chain over finite states $\calS$. Suppose that $P_\pi$ induces a single set of closed irreducible states $\calS_C$. Let $\calS_T = \calS \backslash \calS_C$ be the set of transient states. Then the following properties hold:
    \begin{enumerate}
        \item There exists a unique stationary distribution $d_\pi$ for $P_\pi$ such that $d_\pi(s) >0$ for $s\in \calS_C$ and $d_\pi(s) = 0$ for $s \in \calS_T$.
        \item Let $\mu_t(s) = \bbE_{P_\pi}[\mathds{1}\{s_t = s\} | s_1 \sim \mu_1]$, where $\mu_1$ is the given initial state distribution. Then we have $\lim_{n\to \infty} \frac{1}{n}\sum_{t=1}^n \mu_t(s) = d_\pi(s)$ for all $s \in \calS$.
    \end{enumerate}
\end{lemma}

\begin{proof}

For simplicity, with some abuse of notation, let $P_\pi \in \bbR^{|\calS|\times|\calS|}$ denote the transition matrix, where $P_{\pi, ij}$ denotes the probability of transition from state $i$ to state $j$. Also, we use $d_\pi,\mu_t \in \bbR^{|\calS|}$ for the vector notation of the stationary distribution and the expected state distribution at $t$, respectively. 
Moreover, let $d_{\pi,\calS_C}, \mu_{t,\calS_C} \in \bbR^{|\calS_C|}$ denote subvectors such that $d_{\pi,\calS_C}(s) = d_\pi(s)$ and $\mu_{t,\calS_C} = \mu_{t}(s)$ for all $s\in \calS_C$. A similar rule is applied to $d_{\pi,\calS_T}, \mu_{t,\calS_T} \in \bbR^{|\calS_T|}$. 
Let $\bm{0}_{\calS}, \bm{1}_{\calS} \in \bbR^{|\calS|},\ \bm{0}_{\calS_C}, \bm{1}_{\calS_C} \in \bbR^{|\calS_C|}$, and $ \bm{0}_{\calS_T}, \bm{1}_{\calS_T} \in \bbR^{|\calS_T|}$ denote all-0 and all-1 vectors, respectively.

\vspace{5mm}\emph{Proof of the first statement. }
Without loss of generality, we can express $P_\pi$ as the following block matrix:
\begin{align*}
    P_\pi = \begin{bmatrix}
        M_C & 0 \\ M & M_T
    \end{bmatrix}
\end{align*}
where $M_T$ is the transition matrix between states in $\calS_T$, $M$ is from $\calS_T$ to $\calS_C$, and $M_C$ is between states in $\calS_C$. Note that a stationary distribution $d_\pi \in \bbR_+^{|\calS|}$ satisfies $P_\pi^\top d_\pi = d_\pi$ and $\bm{1}_{\calS}^\top d_\pi = 1$. Here, the condition $P_\pi^\top d_\pi = d_\pi$ can be rewritten as
\begin{align*}
    & (I - M_T^\top) d_{\pi,\calS_T} = \bm{0}_{\calS_T},
    \quad 
    M^\top d_{\pi,\calS_T} + M_C^\top d_{\pi,\calS_C} = d_{\pi,\calS_C}.
\end{align*}
It is known that $I-M_T$ is invertible, when $M_T$ is a transition matrix for transient states (Proposition A.3 in \citet{puterman1994markov}). Therefore, we have $d_{\pi,\calS_T} = \bm{0}_{\calS_T}$. Now, we focus on $d_{\pi,\calS_C}$ such that $M_C^\top d_{\pi,\calS_C} = d_{\pi,\calS_C}$ and $\bm{1}_{\calS_C}^\top d_{\pi,\calS_C} = 1$. Note that it has a unique solution that is $d_{\pi,\calS_C}(s) > 0$ for all $s\in \calS_C$ (Theorem 1.7.7 in \citet{norris1998markov}). This concludes the proof of the first statement.

\vspace{5mm}\emph{Proof of the second statement. }
First, we show that $\lim_{n\to\infty} \frac{1}{n}\sum_{t=1}^n\mu_t(s) = 0$ for $s\in \calS_T$. Note that $\mu_t = (P_\pi^\top)^{t-1}\mu_1$. Since $M_T$ is for transient states, its spectral radius is strictly less than $1$ (Proposition A.3 in \citet{puterman1994markov}). This leads to, as $t \to \infty$,
\begin{align}\label{eq:lem:MC: statement 2 transient}
    \lim_{t \to \infty} \mu_{t,\calS_T} = \lim_{t \to \infty} (M_T^\top)^{t-1}\mu_{1,\calS_T} \ = \ \bm{0}_{\calS_T}.
\end{align}
This implies that the limit of its average satisfies $\lim_{n\to\infty} \frac{1}{n}\sum_{t=1}^{n}\mu_t(s) = 0$ for all $s\in \calS_T$. Next, we consider states in $\calS_C$. Let $\bar\mu_t(s) = \mathds{1}\{s_t =s\}$. Let $\tau = \inf\{t \geq 1: s_t \in \calS_C\}$ denote the time of the first visit $\calS_C$. Since $\calS_T$ is a transient set and $\calS_C$ is the unique recurrent class, the chain visits $\calS_T$ only finite times a.s., thus, $\tau < \infty$ a.s. For $t \geq\tau$, since the transition only occurs in $\calS_C$, we can consider the Markov chain restricted to $\calS_C$, which is irreducible. Note that Theorem 1.10.2 in \citet{norris1998markov} implies that for $s\in \calS_C$,
\[
    \lim_{n\to \infty}\frac{1}{n-\tau + 1} \sum_{t=\tau}^n \bar\mu_{t}(s) \ = \ d_{\pi}(s) \quad \textnormal{a.s.}
\]
This leads to
\begin{align*}
    \lim_{n\to \infty} \frac{1}{n}\sum_{t=1}^n \bar \mu_{t}(s) = \lim_{n\to \infty} \left(\frac{1}{n}\sum_{t=1}^{\tau-1} \bar \mu_{t}(s) + \frac{n-\tau+1}{n}\frac{1}{n-\tau+1}\sum_{t=\tau}^n \bar \mu_{t}(s)\right) = d_{\pi}(s) \quad \textnormal{a.s.}
\end{align*}
Moreover, we know that $|\frac{1}{n}\sum_{t=1}^n \bar\mu_{t,\calS_C}| \leq 1$. Then by the bounded convergence theorem (Theorem 1.5.3 in \citet{durrett2019probability}), it follows that for $s\in \calS_C$,
\begin{align*}
    \lim_{n\to \infty} \frac{1}{n}\sum_{t=1}^n\bbE_{P_\pi}\left[\bar\mu_t(s) | s_1 \sim \mu_1\right] = d_{\pi}(s).
\end{align*}
Since $\bbE_{P_\pi}[\bar\mu_t(s)|s_1 \sim \mu_1] = \mu_t(s)$, this implies that $\lim_{n\to \infty} \frac{1}{n} \sum_{t=1}^n\mu_{t,\calS_C} = d_{\pi,\calS_C}$. By combining this with \eqref{eq:lem:MC: statement 2 transient}, we complete the proof of the second statement.
\end{proof}

\begin{lemma}\label{lem:hat P -> P}
    For $n \in \bbN$, let $r,g \in \bbR^n$ and $b \in \bbR$. Let $\calP\subseteq \bbR^n$ be a non-empty polytope. Let $\widehat\calP \subseteq \calP$ denote a subset of $\calP$, which is possibly not a polytope. Suppose that $\{q\in \widehat\calP: g^\top q \geq b\} \neq \varnothing$ and $\relint(\calP) \cup \ext(\calP) \subseteq \widehat\calP$. Then we have
    \begin{align}\label{eq:lem:hat P -> P}
    \begin{aligned}
        \sup_{q\in \widehat\calP}\quad & r^\top q \\
         \textnormal{s.t.} \quad& g^\top q \geq b 
    \end{aligned}
    \quad= \quad
    \begin{aligned}
        \max_{q\in \calP}\quad & r^\top q \\
         \textnormal{s.t.} \quad& g^\top q \geq b.
    \end{aligned}
    \end{align}
\end{lemma}
\begin{proof}
    Let $F_\calP, F_{\widehat\calP}$ denote the feasible sets of each optimization problems, i.e.,
    \[
        F_\calP = \{q \in \calP: g^\top q \geq b\}, \ F_{\widehat\calP} = \{q \in \widehat\calP: g^\top q \geq b\}.
    \]
    Since $F_{\widehat\calP} \subseteq F_\calP$ and $F_{\widehat\calP}$ is non-empty by the assumption, we know that $F_\calP$ is a non-empty polytope. Then the right-hand side of \eqref{eq:lem:hat P -> P} becomes an LP, and its maximum is attainable by some $q^* \in \calP$. We define $P^*$ as
    \begin{equation*}
        P^* = r^\top q^* = \max_{q\in \calP}\ r^\top q \quad\textnormal{ s.t. }\quad g^\top q\geq b.
    \end{equation*}
    Since $F_{\widehat\calP}\subseteq F_\calP$, it is clear that 
    \[
        P^* \geq \sup_{q\in \widehat\calP}\ r^\top q \quad\textnormal{ s.t. }\quad g^\top q\geq b.
    \]
    Now, it is sufficient to prove that $P^* \leq \sup_{q\in \widehat\calP}\ r^\top q \quad\textnormal{ s.t. }\quad g^\top q\geq b.$

    \paragraph{Case 1: $\exists\bar q \in \{q\in\relint(\calP): g^\top q > b\}$.}
    In this case, for $\theta\in(0,1)$, we define
    \[
        \bar q_\theta = \theta\bar q + (1-\theta) q^*.
    \]
    Since $q^*\in\calP$, $\bar q \in \relint(\calP)$, we have $\bar q_\theta \in \relint(\calP)\subseteq \widehat\calP$ for $\theta \in (0,1)$ (Theorem 6.1 in \citet{rockafellar1970convex}). Moreover, it is clear that $g^\top \bar q_\theta > b$ for $\theta\in (0,1)$. Therefore, $\bar q_\theta \in F_{\widehat\calP}$ for $\theta\in(0,1)$. By taking $\theta \to 0+$, we have
    \begin{align*}
        \sup_{q \in F_{\widehat\calP}} r^\top q \geq \lim_{\theta \to 0+} r^\top \bar q_\theta = r^\top q^*.
    \end{align*}

    \paragraph{Case 2: $\{q\in\relint(\calP): g^\top q > b\} = \varnothing$.}
    Note that this implies that $\relint(\calP) = \{q \in \relint(\calP): g^\top q \leq b\}$. Based on this fact,
    \begin{align*}
        \calP = \overline{\relint(\calP)} = \overline{\{q\in \relint(\calP): g^\top q\leq b\}} \subseteq \{q \in \overline{\relint(\calP)}: g^\top q\leq b\} = \{q \in \calP: g^\top q\leq b\}
    \end{align*}
    where the second last relation follows from $\overline{A \cap B} \subseteq \overline A \cap \overline B$. This implies that $\calP = \{q\in\calP:g^\top q \leq b\}$. Moreover, it follows that
    \[
        F_{\calP} = \calP \cap \{q\in\bbR^n: g^\top q \geq b\} = \{q\in\calP: g^\top q = b\}.
    \]
    Thus, $F_{\calP}$ is a face of the polytope $\calP$, which is also a polytope. Then the maximizer of $\max_{q\in F_{\calP}} r^\top q$, denoted by $q^*$, is attained at an extreme point of $F_{\calP}$, i.e., $q^* \in \ext(F_{\calP})$. Additionally, we know that extreme points of a face of a polytope are also extreme points of that polytope, i.e., $\ext(F_{\calP}) \subseteq \ext(\calP)$. These imply that $q^* \in \ext(\calP)\subseteq \widehat\calP$. Note that $g^\top q^* = b$. Therefore, we have
    \[
        \sup_{q\in F_{\widehat\calP}}\ r^\top q \geq r^\top q^* = P^*.
    \]
    In both cases, we prove $\sup_{q\in F_{\widehat\calP}}\ r^\top q \geq P^*.$ This completes the proof.
\end{proof}

\begin{lemma}\label{lem:relint}
    Let $\calM = (\calS,\calA, P, r)$ be a weakly communicating MDP, and let $s_1 \in\calS$ be any initial state. For any $q^* \in \relint(\calP)$, there exists $\pi^*\in \PiSR$ that satisfies the following: (1) $P_{\pi^*}$ is unichain, (2) $q^* = \occ{\pi^*}{s_1}$.
\end{lemma}
\begin{proof}
The outline of the proof is as follows. Given $q^* \in \relint(\calP)$, we construct the stationary randomized policy $\pi^*$. Next, based on properties of relative interior and weakly communicating MDPs, we show $P_{\pi^*}$ is unichain. Finally, we prove that the occupation measure $\pi^*$ is equal to $q^*$, i.e., $\occ{\pi^*}{s_1} = q^*$.

\emph{Step 1: Construct a stationary policy $\pi^*$.} 
    We define $\pi^*$ as 
    \begin{align*}
        \pi^*(a|s) = \frac{q^*(s,a)}{\sum_{a}q^*(s,a)} \quad \textnormal{ if }\quad \sum_{a}q^*(s,a) > 0.
    \end{align*}
    Otherwise, for $s$ such that $\sum_a q^*(s,a) = 0$, take an arbitrary action. Then, we define
    \[
        d^*(s) = \sum_{a} q^*(s,a).
    \]    
Next, we introduce a useful property that $d^*(s)$ is a stationary distribution under $P_{\pi^*}$. This is because for any $s\in\calS$,
    \begin{align*}
        \sum_{s'}P_{\pi^*}(s|s')d^*(s') 
        &= \sum_{s'}\sum_{a'}P(s|s',a')\pi^*(a'|s')d^*(s')\\
        &= \sum_{s': d^*(s')>0}\sum_{a'}P(s|s',a')\pi^*(a'|s')d^*(s') + \sum_{s': d^*(s')=0}\sum_{a'}P(s|s',a')\pi^*(a'|s')d^*(s')\\
        &= \sum_{s': d^*(s')>0}\sum_{a'}P(s|s',a')\frac{q^*(s',a')}{d^*(s')}d^*(s') + \sum_{s': d^*(s')=0}\sum_{a'}P(s|s',a')\pi^*(a'|s')d^*(s')\\
        &= \sum_{s': d^*(s')>0}\sum_{a'}P(s|s',a')\frac{q^*(s',a')}{d^*(s')}d^*(s') + \sum_{s': d^*(s')=0}\sum_{a'}P(s|s',a')q^*(s',a') \\
        &= \sum_{s'}\sum_{a'}P(s|s',a')q^*(s',a')\\
        &= \sum_{a}q^*(s,a) \\
        &= d^*(s)
    \end{align*}
    where the third equality is due to the definition of $\pi^*$, the fourth equality follows from $\pi^*(a'|s')d^*(s') = q^*(s',a')=0$ for all $s'$ such that $d^*(s') = 0$, and the second last equality follows from $q^* \in \calP$.
    
    \emph{Step 2: Show $P_{\pi^*}$ is unichain.}
    Now, we show $P_{\pi^*}$ has a set of single closed irreducible states with transient states. 

    Since $\calM$ is a weakly communicating MDP, there exists a stationary randomized policy $\check\pi$ whose induced Markov chain $P_{\check\pi}$ has a set of single closed irreducible states, denoted by $\calS_C$. Moreover, $\calS_T = \calS \backslash \calS_C$ is the set of states that are transient under all stationary policies. By \Cref{lem:MC}, there exists a unique stationary distribution $d_{\check\pi}$ such that $d_{\check\pi}(s) > 0$ for $s\in \calS_C$ and $d_{\check\pi}(s) = 0$ for $s \in \calS_T$.
    Based on this, we define
    \begin{align*}
        q^{\check\pi}(s,a) = \check\pi(a|s)d_{\check\pi}(s).
    \end{align*}
    Note that $q^{\check\pi} \in \calP$ can be shown as follows. We have $\sum_{s,a} q^{\check\pi}(s,a) = \sum_s d_{\check\pi}(s) \sum_a \check\pi(a|s) = 1$ and $q^{\check\pi}(s,a)\geq 0$. Moreover, 
    \begin{align*}
        \sum_{s',a'}P(s|s',a')q^{\check\pi}(s',a') = \sum_{s',a'}P(s|s',a')\check\pi(a'|s')d_{\check\pi}(s') = \sum_{s'}P(s|s')d_{\check\pi}(s') = d_{\check\pi}(s)
    \end{align*}
    where the last equality follows from that $d_{\check\pi}$ is a stationary distribution. Furthermore, it is trivial that $d_{\check\pi}(s) = \sum_{a} q^{\check\pi}(s,a)$. These imply that $q^{\check\pi} \in \calP$.

    Consider a Markov chain $P_{\pi^*}$ induced by $\pi^*$. First, we prove that $\calS_C$ is closed irreducible under $P_{\pi^*}$. Since $\calS_C$ is closed irreducible under $P_{\check\pi}$, for any $s', s'' \in \calS_C$, there exists a sequence of states in $\calS_C$, denoted by $(s_1, s_2, \ldots, s_m)$ for some $m \in \bbN$, such that $s' = s_1, s''=s_m$, and $P_{\check\pi}(s_{i+1}|s_i) > 0$ for all $i = 1, \ldots, m-1$. It can be rewritten as $\sum_{a}P(s_{i+1}|s_i, a) \check\pi(a|s_i) >0$ for $i=1,\ldots, m-1$. This implies that $\check\pi (a_{s_i}| s_i) > 0$ and $P(s_{i+1}|s_i,a_{s_i}) > 0$ for some $a_{s_i}$ for all $i=1,\ldots, m-1$. By the definition of $q^{\check\pi}$, for each $i = 1,\ldots, m-1$, we have $q^{\check\pi}(s_i, a_{s_i}) > 0$. By \Cref{lem:relint equiv}, $q^*(s_i,a_{s_i}) > 0$ for all $i=1,\ldots,m-1$. This implies that
    \begin{align*}
        P_{\pi^*}(s_{i+1}|s_i) = \sum_a P(s_{i+1}|s_i,a)\pi^*(a|s_i) = \sum_a P(s_{i+1}|s_i,a)\frac{q^*(s_i,a)}{\sum_a q^*(s_i,a)} > 0.
    \end{align*}
    Since $P_{\pi^*}(s_{i+1}|s_i) > 0$ for all $i =1,\ldots,m-1$, we have any $s', s'' \in \calS_C$ communicate under $P_{\pi^*}$. Moreover, $\calS_C$ is closed. If not, there exist $s \in \calS_C, \ s' \in \calS_T$ such that $s \to s'$. Since $s'$ is transient, so is $s$, as it is known that a state $s$ is transient if $s \to s'$ for some transient state $s'$ (Appendix A.3 in \citet{puterman1994markov}). Since all states in $\calS_C$ communicate, $\calS_C$ becomes transient. This implies that all states in $\calS$ are transient, which contradicts the fact that any finite Markov chain has at least one positive recurrent class. Finally, this implies that $P_{\pi^*}$ induces a single closed irreducible set $\calS_C$ and transient set $\calS_T$.

    \emph{Step 3: Show $\occ{\pi^*}{s_1} = q^*$.}
    Let $d_{\pi^*}$ denote the stationary distribution of $P_{\pi^*}$. Since $P_{\pi^*}$ induces a single recurrent class and a transient set, by \Cref{lem:MC},  for all $s\in \calS$,
    \begin{align*}
        \lim_{T\to \infty}\frac{1}{T} \sum_{t=1}^T\bbE_{P_{\pi^*}} [\mathds 1\{s_t=s\} | s_1] = d_{\pi^*}(s).
    \end{align*}
    Multiplying both sides by $\pi^*(a|s)$ yields
    \begin{align*}
        \lim_{T\to \infty} \frac{1}{T} \sum_{t=1}^T \bbE_{{\pi^*}}[\mathds 1\{s_t=s, a_t=a\}|s_1] = d_{\pi^*}(s)\pi^*(a|s).
    \end{align*}
    To further deduce, we derive the following property: for any $s\in\calS$,
    \begin{align*}
        d_{\pi^*}(s) = d^*(s).
    \end{align*}
    Note that this property follows from that (i) by \Cref{lem:MC}, $P_{\pi^*}$ has the unique stationary distribution $d_{\pi^*}(s)$ and (ii) by the argument in Step 1, $d^*(s) = \sum_a q^*(s,a)$ is also a stationary distribution of $P_{\pi^*}$. This leads to for all $(s,a) \in \calS \times \calA$,
    \[
        d^*(s)\pi^*(a|s) = q^*(s,a),
    \]
    since for $s$ such that $d^*(s) >0$ we have $d^*(s)\pi^*(a|s) = d^*(s)\frac{q^*(s,a)}{d^*(s)} = q^*(s,a)$ , and for $s$ such that $d^*(s)=0$ we have $d^*(s)\pi^*(a|s) = q^*(s,a) = 0$.
    
    Finally, we show that
    \begin{align*}
        \lim_{T\to \infty} \frac{1}{T} \sum_{t=1}^T \bbE_{{\pi^*}}[\mathds 1\{s_t=s, a_t=a\}|s_1] = q^*(s,a).
    \end{align*}
    Recall that the left-hand side is the definition of $\occ{\pi^*}{s_1}(s,a)$. This implies that for any $q^*\in \relint(\calP)$, there exists a stationary policy $\pi^*$ whose occupation measure satisfies $\occ{\pi^*}{s_1} = q^*$.
\end{proof}

\subsection{Proof of Lemma~\ref{lem:relint & ext}}\label{subsec:proof of lem:relint & ext}

\textbf{Proof of $\bm{ \calPSR{s_1} \subseteq \calP}. \quad$} 
The statement is a direct consequence of Proposition 3.1 in \citet{mannor2005empirical}.

\textbf{Proof of $\bm{\relint(\calP) \subseteq \calPSR{s_1}}. \quad$} 
By \Cref{lem:relint}, for any $q^* \in \relint(\calP)$, there exists $\pi^* \in \PiSR$ such that $q^* = \occ{\pi^*}{s_1}$. Since $\occ{\pi^*}{s_1}$ is an occupation measure induced by stationary policy $\pi^*$ starting from $s_1$, $\occ{\pi^*}{s_1} \in \calPSR{s_1}$. This concludes the proof.

\textbf{Proof of $\bm{\ext(\calP) \subseteq \calPSR{s_1}}. \quad$}
    Fix $q^* \in \ext(\calP)$. Then there exists a vector $r \in \bbR^{|\calS|\times |\calA|}$ such that $\argmax_{q \in \calP} r^\top q = \{q^*\}$. For such $r$, we consider the following MDP:
    \begin{align*}
        \sup_{\pi\in\PiSR} J_r^{\pi}(s_1).
    \end{align*}
    It can be rewritten as
    \begin{align*}
        \sup_{q\in\calPSR{s_1}} r^\top q.
    \end{align*}
    Note that since $\relint(\calP)$ is the relative interior of $\calP$, 
    \begin{align*}
        \sup_{q\in\relint(\calP)}r^\top  q = \max_{q \in \calP} r^\top q = r^\top q^*.
    \end{align*}
    We know that $\relint(\calP) \subseteq \calPSR{s_1}$ and $\calPSR{s_1} \subseteq \calP$. This implies that
    \begin{align*}
        r^\top q^*= \sup_{q \in \relint(\calP)} r^\top q \leq \sup_{q \in \calPSR{s_1}} r^\top q \leq \sup_{q\in\calP}r^\top q = r^\top q^*.
    \end{align*}
    This implies that
    \begin{align*}
        \sup_{\pi\in\PiSR} J_r^\pi(s_1) = \sup_{q \in \calPSR{s_1}} r^\top q = r^\top q^*.
    \end{align*}
    By \Cref{lem: J_r = sup J_r}, there exists $\pi^* \in \PiSD$ such that $J_r^{\pi^*}(s_1) = \sup_{\pi \in \PiSR}J_r^\pi(s_1)$. This implies that $r^\top \occ{\pi^*}{s_1} = J_r^{\pi^*}(s_1) = r^\top q^*$. Moreover, we have $\occ{\pi^*}{s_1} \in \calPSR{s_1} \subseteq \calP$. Then $\occ{\pi^*}{s_1}$ is also an optimal solution to $\max_{q\in\calP} r^\top q$. However, we pick $r$ that induces a unique optimal solution, it follows that $\occ{\pi^*}{s_1} = q^*$. Therefore, any $q^* \in \ext(\calP)$ is achievable by some $\pi^* \in \PiSR$, i.e.,
    \begin{equation}\label{eq:ext}
        \ext(\calP) \subseteq \calPSR{s_1}.
    \end{equation}


\subsection{Proof of Theorem~\ref{thm:strong duality}}\label{subsec:proof of thm:strong duality}

Recall that 
\[
    \calP=\Biggr\{q \in \bbR_+^{|\calS|\times |\calA|}: \sum_{s,a} q(s,a) = 1,\ \sum_{s',a'}P(s|s',a')q(s',a') = \sum_{a}q(s,a) \ \forall s\Biggl\}
\]
and $\calPSR{s_1}$ is the set of occupation measures induced by stationary policies. By \Cref{lem:relint & ext}, 
\begin{align}\label{eq:strong duality int+ext}
    \relint(\calP) \cup \ext(\calP) \subseteq \calPSR{s_1} \subseteq \calP.
\end{align}
Based on this, we introduce a well-known fact from the linear programming literature. Since $\calP$ is a non-empty polytope and $\textnormal{ext}(\calP) \subseteq \calPSR{s_1}$, for any bounded function $\ell: \calS\times \calA \to \bbR$
\begin{align}\label{eq:any}
    \sup_{q \in \calPSR{s_1}} \ell^\top q = \max_{q \in \calPSR{s_1}} \ell^\top q = \max_{q\in \calP} \ell^\top q.
\end{align}
By the definition of $\calPSR{s_1}$, the primal problem \eqref{eq:average-reward CMDP} can be rewritten as
\begin{equation}\label{eq:strong duality 2}
\begin{array}{ccc}
\begin{aligned}
    \sup_{\pi\in\PiSR}\quad & J_r^\pi(s_1) \\
    \text{s.t.}\quad & J_g^\pi(s_1) \geq b
\end{aligned}
&\quad=\quad&
\begin{aligned}
    \sup_{q \in \calPSR{s_1}}\quad & r^\top q  \\
    \text{s.t.}\quad &  g^\top q \geq b
\end{aligned}
\end{array}
\end{equation}
Note that \eqref{eq:strong duality int+ext} and that $\{q \in \calPSR{s_1}: g^\top q \geq b\}$ is non-empty by the assumption. Then by \Cref{lem:hat P -> P},
\begin{equation*}
    \begin{array}{ccc}
\begin{aligned}
    \sup_{q \in \calPSR{s_1}}\quad & r^\top q  \\
    \text{s.t.}\quad &  g^\top q \geq b
\end{aligned}
&\quad=\quad&
\begin{aligned}
    \max_{ q \in \calP}\quad & r^\top q  \\
    \text{s.t.}\quad &  g^\top q \geq b
\end{aligned}
\end{array}
\end{equation*}
Then we observe that the right-hand side is an LP, as $\calP$ is a polytope. Moreover, we know that its feasible region $\{q\in \calP: g^\top q \geq b\}$ is non-empty, because $\{q\in \calPSR{s_1}: g^\top q \geq b\}$ is non-empty and $\{q\in \calPSR{s_1}: g^\top q \geq b\} \subseteq \{q\in \calP: g^\top q \geq b\}$. Then by the strong duality of LP (\Cref{lem:LP strong duality}), $\lambda^* \in \argmin_{\lambda\geq 0} \max_{q \in \calP} r^\top q + \lambda(g^\top q -b)$ and
\begin{equation}\label{eq:strong duality 5}
\begin{array}{ccc}
\begin{aligned}
    \max_{ q \in \calP}\quad & r^\top q  \\
    \text{s.t.}\quad &  g^\top q \geq b
\end{aligned}
&\quad=\quad&
\begin{aligned}
    \min_{\lambda\geq0}\max_{ q \in \calP}\quad & r^\top q + \lambda(g^\top q - b).
\end{aligned}
\end{array}
\end{equation}
By \eqref{eq:any}, for any $\lambda\geq 0$, we have $\max_{q \in \calP}r^\top q + \lambda(g^\top q - b) = \sup_{q \in \calPSR{s_1}} r^\top q + \lambda (g^\top q - b)$. Thus, the minimizer $\lambda^*$ to the left-hand side is also a minimizer to the right-hand side. Taking $\min_{\lambda\geq0}$ on both sides yields
\begin{equation}\label{eq:strong duality 6}
    \min_{\lambda\geq0}\max_{ q \in \calP}\  r^\top q + \lambda(g^\top q - b) = \min_{\lambda\geq0} \sup_{q \in \calPSR{s_1}} r^\top q + \lambda (g^\top q - b)
\end{equation}
and 
\begin{equation}\label{eq:strong duality lambda min}
    \lambda^* \in \argmin_{\lambda\geq0}\sup_{q \in \calPSR{s_1}} r^\top q + \lambda (g^\top q - b).
\end{equation}
By the definition of $\calPSR{s_1}$, we have
\begin{equation}\label{eq:strong duality 7}
\begin{array}{ccc}
\begin{aligned}
    \min_{\lambda\geq0} \sup_{\widehat q \in \calPSR{s_1}} r^\top \widehat q + \lambda (g^\top \widehat q - b)
\end{aligned}
&\quad=\quad&
\begin{aligned}
\min_{\lambda\geq0}\sup_{\pi\in\PiSR}  J_r^\pi(s_1) + \lambda(J_g^\pi(s_1) - b)
\end{aligned}
\end{array}
\end{equation}
The existence of optimal dual variable is guaranteed by \eqref{eq:strong duality lambda min}. Moreover, by \eqref{eq:strong duality 2} - \eqref{eq:strong duality 7}, we prove that
\[
        \begin{array}{ccc}
             \begin{aligned}
                 \sup_{\pi \in \PiSR} \quad& J_r^\pi(s_1) \\
                 \textnormal{s.t.} \quad& J_g^\pi(s_1) \geq b
             \end{aligned}
             &=&
             \min_{\lambda\geq0}\sup_{\pi\in\PiSR}  J_r^\pi(s_1) + \lambda(J_g^\pi(s_1) - b).
        \end{array}
        \]

\section{Missing proofs for Section~\ref{sec:average-reward linear CMDP}}\label{appendix:Regret analysis}

\subsection{Proof of Lemma~\ref{lem:optimal dual variable bound}} 
Since the primal problem~\eqref{eq:average-reward CMDP} is assumed to be feasible, we can use \Cref{thm:strong duality}. Thus, by strong duality, there exists $\lambda^* \in \argmin_{\lambda\geq 0} D(\lambda)$, where $D(\lambda) = \sup_{\pi\in\PiSR}\ J_r^\pi(s_1) + \lambda(J_g^\pi(s_1)-b)$. Moreover, $J_r^* = D(\lambda^*)$, where $J_r^*(s_1) = \sup_{\pi \in \PiSR} J_r^\pi(s_1) \ \text{s.t} \ J_g^\pi(s_1) \geq b$. Then it follows that
\begin{align*}
    J_r^*(s_1) = D(\lambda^*) 
    &= \sup_{\pi\in\PiSR} J_r^\pi(s_1) + \lambda^*(J_g^\pi(s_1)-b)\\
    &\geq J_r^{\bar\pi}(s_1) + \lambda^*(J_g^{\bar\pi}(s_1) - b) \\
    &\geq J_r^{\bar\pi}(s_1) + \lambda^*\slater
\end{align*}
where $\bar \pi$ is the Slater policy satisfying $J_g^{\bar\pi}(s_1) \geq b+ \slater$. Since $0 \leq J_r^*(s_1), J_r^{\bar\pi}(s_1) \leq 1$, it can be rewritten as
\begin{align*}
    \lambda^* \leq \frac{J_r^*(s_1) - J_r^{\bar\pi}(s_1)}{\slater} \leq \frac{1}{\slater}.
\end{align*}

\subsection{Proof of Lemma~\ref{lem:good event}}

Recall that
\begin{align*}
    &\bfw_{h+1}^k = \bmSigma_k^{-1} \sum_{\tau=1}^{k-1} \sum_{j=1}^H \phi(s_j^\tau, a_j^\tau) (V_{h+1}^k(s_{j+1}^\tau) - \min_{s'} V_{h+1}^k(s'))\\
    &\bfw_{h+1}^{k,*} = \sum_{s\in \calS} \psi(s)(V_{h+1}^k(s) - \min_{s'}V_{h+1}^k(s')).
\end{align*}
By the definition of linear MDP, we have $\phi(s,a)^\top \bfw_{h+1}^{k,*} = PV_{h+1}^k(s,a) - \min_{s'}V_{h+1}^k(s')$, as $\sum_{s'\in \calS} \phi(s,a)^\top \psi(s') = \sum_{s'}P(s'|s,a) =1$. Note that $V_{h+1}^k(s) - \min_{s'} V_{h+1}^k(s') \leq \spn(V_{h+1}^k) \leq 2\zeta$ for all $s\in\calS$. Then, by Lemmas \ref{lem:hong2024 lemma 12} and \ref{lem:hong2024 lemma 13}, we have
\begin{align}\label{eq:bfw l2-norm bound}
    \|\bfw_{h+1}^k\|_2 \leq 2\zeta\sqrt{dT},\quad \|\bfw_{h+1}^{k,*}\|_2 \leq 2\zeta\sqrt{d}.
\end{align}

Fix $(h,k) \in [H] \times [K]$. For simplicity, let $\bar V_{h+1}^k(s) = V_{h+1}^k(s) - \min_{s'} V_{h+1}^k(s')$. Note that
\begin{align*}
    &\left|\phi(s,a)^\top (\bfw_{h+1}^k - \bfw_{h+1}^{k,*})\right|\\
    &= \left|\phi(s,a)^\top\bmSigma_k^{-1} \left(\sum_{\tau=1}^{k-1}\sum_{j=1}^H \phi(s_j^\tau, a_j^\tau)(V_{h+1}^k(s_{j+1}^\tau) - \min_{s'} V_{h+1}^k(s')) - \bmSigma_k \bfw_{h+1}^{k,*}\right)\right|\\
    &\leq \|\phi(s,a)\|_{\bmSigma_k^{-1}} \left\|\sum_{\tau=1}^{k-1}\sum_{j=1}^H \phi(s_j^\tau,a_j^\tau) \left((V_{h+1}^k(s_{j+1}^\tau) - \min_{s'}V_{h+1}^k(s')) - \phi(s_j^\tau,a_j^\tau)^\top \bfw_{h+1}^{k,*}\right)\right\|_{\bmSigma_k^{-1}} \\
    &\quad+ \left|\phi(s,a)^\top \bmSigma_k^{-1} \bfw_{h+1}^{k,*}\right|\\
    &= \|\phi(s,a)\|_{\bmSigma_k^{-1}} \left\|\sum_{\tau=1}^{k-1}\sum_{j=1}^H \phi(s_j^\tau,a_j^\tau) \left(\bar V_{h+1}^k(s_{j+1}^\tau) - P\bar V_{h+1}^k(s_j^\tau, a_j^\tau))\right)\right\|_{\bmSigma_k^{-1}} \\
    &\quad+ \left|\phi(s,a)^\top \bmSigma_k^{-1} \bfw_{h+1}^{k,*}\right|
\end{align*}
where the inequality follows from the fact that $\bmSigma_k = I + \sum_{\tau=1}^{k-1}\sum_{j=1}^H \phi(s_j^\tau,a_j^\tau)\phi(s_j^\tau,a_j^\tau)^\top$ and the triangle inequality. By \eqref{eq:bfw l2-norm bound}, we have
\begin{align*}
    \left|\phi(s,a)^\top \bmSigma_k^{-1} \bfw_{h+1}^{k,*}\right| \leq \|\bfw_{h+1}^{k,*}\|_{\bmSigma_k^{-1}}\|\phi(s,a)\|_{\bmSigma_k^{-1}} \leq 2\zeta\sqrt{d}\|\phi(s,a)\|_{\bmSigma_k^{-1}}.
\end{align*}
By \Cref{lem:concentration}, with probability at least $1-\delta / T$,
\begin{align*}
    &\left\|\sum_{\tau=1}^{k-1}\sum_{j=1}^H \phi(s_j^\tau,a_j^\tau) \left(\bar V_{h+1}^k(s_{j+1}^\tau) - P\bar V_{h+1}^k(s_j^\tau, a_j^\tau))\right)\right\|_{\bmSigma_k^{-1}}^2 \\
    &\leq 16\zeta^2 \left[\frac{d}{2}\log(T+1) + \log\frac{\calN_\epsilon T}{\delta}\right] + 8T^2\epsilon^2.
\end{align*}
where $\calN_\epsilon$ denotes the $\epsilon$-covering number of the value function class, which is the set of all possible $\bar V_{h+1}^k$. Note that the centering mapping $V(\cdot) \mapsto V(\cdot) - \min_{s'}V(s')$ is $2$-Lipschitz. Therefore, this transformation affects the covering number only up to constant factors. Then, to bound $\log(\calN_\epsilon)$, we first deduce that
\begin{align*}
    r(s,a) + \lambda_k g(s,a) + \phi(s,a)^\top \bfw_{h+1}^k = \phi(s,a)^\top \left(\theta_r + \lambda_k \theta_g + \bfw_{h+1}^k\right).
\end{align*}
It follows that
\begin{align*}
    \|\theta_r + \lambda_k \theta_g + \bfw_{h+1}^k\|_2 \leq \sqrt{d} + \frac{2}{\slater}\sqrt{d} + 2\zeta\sqrt{dT} \leq \frac{3}{\slater}\sqrt{d} + 2\zeta\sqrt{dT}.
\end{align*}
Then, by \Cref{lem:hong2024 lemma 15}, 
\begin{align*}
    \log(\calN_\epsilon) 
    &\leq d\log\left(1 + \frac{8\sqrt{d}(3/\slater + 2\zeta\sqrt{T})}{\epsilon}\right) + \log(1+4H(1+2/\gamma)/\epsilon) \\
    &\quad+ d^2\log\left(1 + 8d^{1/2}\beta^2/\epsilon^2\right).
\end{align*}
It follows that
\begin{align*}
    &\left\|\sum_{\tau=1}^{k-1}\sum_{j=1}^H \phi(s_j^\tau,a_j^\tau) \left(\bar V_{h+1}^k(s_{j+1}^\tau) - P\bar V_{h+1}^k(s_j^\tau, a_j^\tau))\right)\right\|_{\bmSigma_k^{-1}}^2 \\
    &\leq 16\zeta^2 \Biggl[\frac{d}{2}\log(T+1) + d\log\left(1 + \frac{8\sqrt{d}(3/\slater + 2\zeta\sqrt{T})}{\epsilon}\right) \\
    &\quad + \log(1+4H(1+2/\gamma)/\epsilon) + d^2\log\left(1 + 8d^{1/2}\beta^2/\epsilon^2\right) +\log\frac{T}{\delta}\Biggr] + 8T^2\epsilon^2.
\end{align*}
By taking $\epsilon = d/T$, it follows that
\begin{align*}
    &\left\|\sum_{\tau=1}^{k-1}\sum_{j=1}^H \phi(s_j^\tau,a_j^\tau) \left(\bar V_{h+1}^k(s_{j+1}^\tau) - P\bar V_{h+1}^k(s_j^\tau, a_j^\tau))\right)\right\|_{\bmSigma_k^{-1}}^2 \\
    &\leq 16\zeta^2 \Biggl[\frac{d}{2}\log(2T) + d\log\left(16T^2(3/\slater + 2\zeta) \right) + \log(8H(1+2/\gamma)T) + d^2\log\left(16\beta^2 T^2\right) +\log\frac{T}{\delta}\Biggr] + 8d^2 \\
    &\leq 72\zeta^2 d^2\log\left(16HT^2\beta^2(3/\slater + 2\zeta)/\delta \right) + 8d^2.
\end{align*}
Consider $\beta = c_\beta (\zeta +1)d\log(16HT^2(1/\slater + 2\zeta)/\delta)$. Then there exists $c_\beta = \bigO(1)$ such that
\begin{align*}
    \sqrt{72\zeta^2 d^2\log\left(16HT^2\beta^2(3/\slater + 2\zeta)/\delta \right) + 8d^2} + 2\zeta\sqrt{d} \leq \beta.
\end{align*}
This implies that by taking $\beta = \bigO((\zeta+1)d)$, where $\bigO(\cdot)$ hides polynomial factors in $\log(dT|\calA|\zeta/(\delta\gamma))$, we have with probability at least $1-\delta/T$, for any $(s,a) \in \calS \times\calA$,
\begin{align*}
    |\phi(s,a)^\top(\bfw_{h+1}^{k} - \bfw_{h+1}^{k,*})| \leq  \beta\|\phi(s,a)\|_{\bmSigma_k^{-1}}.
\end{align*}
By taking union bound over all $(h,k) \in [H] \times [K]$, the statement holds for all $(h,k)$, with probability at least $1-\delta$.

\subsection{Proof of Lemma~\ref{lem:optimism}}
Fix $k$. For all $s \in \calS$, we have 
\[
    V_{r,H+1}^{*}(s) + \lambda_k  V_{g,H+1}^{*}(s) = V_{H+1}^k(s) =0.
\]
Assume that $V_{r,h+1}^{*}(s) + \lambda_k  V_{g,h+1}^{*}(s) \leq V_{h+1}^k(s)$. Note that for any $(s,a)$
\begin{align*}
    &Q_{r,h}^{*}(s,a) + \lambda_k  Q_{g,h}^{*}(s,a) \\
    &= r(s,a) + \lambda_k g(s,a) +  P \left(V_{r,h+1}^{*} + \lambda_k  V_{g,h+1}^{*}\right)(s,a) \\
    &\leq r(s,a) + \lambda_k g(s,a) +  P V_{h+1}^k(s,a) \\
    &\leq r(s,a) + \lambda_k g(s,a) + \phi(s,a)^\top\bfw_{h+1}^k+\min_{s'}V_{h+1}^k(s') \\
    &\quad+ |PV_{h+1}^k(s,a)- \min_{s'} V_{h+1}^k(s') -\phi(s,a)^\top\bfw_{h+1}^k| \\
    &= r(s,a) + \lambda_k g(s,a) + \phi(s,a)^\top\bfw_{h+1}^k+\min_{s'}V_{h+1}^k(s') \\
    &\quad+ |\phi(s,a)^\top(\bfw_{h+1}^{k,*} - \bfw_{h+1}^k)| \\
    &\leq r(s,a) + \lambda_k g(s,a) + \phi(s,a)^\top\bfw_{h+1}^k+\min_{s'}V_{h+1}^k(s') + \beta\|\phi(s,a)\|_{\bmSigma_k^{-1}} \\
    &\leq \left(r(s,a) + \lambda_k g(s,a) + \phi(s,a)^\top\bfw_{h+1}^k+\min_{s'}V_{h+1}^k(s') + \beta\|\phi(s,a)\|_{\bmSigma_k^{-1}}\right)\wedge H(1+\frac{2}{\slater}) \\
    &= Q_{h}^k(s,a)
\end{align*}
where the second inequality is due to the triangle inequality, the second equality follows from the definition of $\bfw_{h+1}^{k,*}$, the third inequality follows from \Cref{lem:good event}, and the fourth inequality follows from the fact that $Q_{r,h}^{*}(s,a) + \lambda_k  Q_{g,h}^{*}(s,a) \leq H(1+2/\gamma)$. Moreover,
\[
    V_{r,h}^{*}(s) + \lambda_k  V_{g,h}^{*}(s) \leq \max_{a} \ Q_{r,h}^{*}(s,a) + \lambda_k  Q_{g,h}^{*}(s,a) \leq \max_a \ Q_{h}^k(s,a) = \widetilde V_{h}^k(s).
\]
Note that $J_r^{*}(s), J_g^{*}(s)$ are constants in $s \in \calS$. Then, by \Cref{lem:chen2022 lemma 4}, we have for any $s,s' \in \calS$ and $h\in [H]$,
\[
    |V_{r,h}^{*}(s) - V_{r,h}^{*}(s')| \leq |V_{r,h}^{*}(s) - (H-h+1)J_r^{*}| + |V_{r,h}^{*}(s') - (H-h+1)J_r^{*}| \leq 2\spn(v_r^*).
\]
This implies that, for all $h\in [H]$, $\spn(V_{r,h}^{*}) \leq 2\spn(v_r^*)$, and similarly, we have $\spn(V_{g,h}^{*}) \leq 2\spn(v_g^*)$. Moreover, it follows that $\spn(V_{r,h}^{*} + \lambda_k  V_{g,h}^{*}) \leq \spn(V_{r,h}^{*}) + \lambda_k  \spn(V_{g,h}^{*}) \leq 2\spn(v_r^*) + \frac{4}{\slater} \spn(v_g^*) = 2\zeta$. Then we have
\begin{align*}
    V_{r,h}^{*}(s) + \lambda_k  V_{g,h}^{*}(s) \leq \widetilde V_{h}^k(s) \wedge (\min_{s'}\widetilde V_{h}^k(s') + 2\zeta) = V_{h}^k(s).
\end{align*}
By induction, we have for all $(s,h,k)$,
\begin{align*}
     V_{r,h}^{*}(s) + \lambda_k  V_{g,h}^{*}(s) \leq V_{h}^k(s)
\end{align*}
as desired.

\subsection{Decomposition}

In this subsection, we provide the detailed decomposition~\eqref{eq:decomp} of $\Regret(T) + \lambda \Violation(T)$, which is introduced in \Cref{subsec:analysis}. Throughout the decomposition procedure, we assume \Cref{lem:good event}, which holds with high probability. Then $\Regret(T) + \lambda \Violation(T)$ can be decomposed as follows.
\begin{align}\label{eq:decomp main}
    &\sum_{t=1}^T \left(J_r^*(s_1)  - r(s_t,a_t)\right) + \lambda \sum_{t=1}^T (b-g(s_t,a_t)) \\
    &= \sum_{k=1}^K\sum_{h=1}^H \left(J_r^*(s_1)  - r(s_h^k,a_h^k)\right) + \lambda \sum_{k=1}^K\sum_{h=1}^H (b-g(s_h^k,a_h^k)) \notag\\
    &\leq \sum_{k=1}^K\underbrace{\sum_{h=1}^H \left[\left(J_r^{*}(s_1)  - r(s_h^k,a_h^k)\right) + \lambda_k  \left(J_g^{*}(s_1)- g(s_h^k,a_h^k)\right)\right]}_{\textnormal{(I)}} + \underbrace{\sum_{k=1}^K (\lambda-\lambda_k )\left(Hb-\sum_{h=1}^H g(s_h^k,a_h^k)\right)}_{\textnormal{(II)}}\notag
\end{align}
where the inequality follows from the fact that $J_g^*(s_1) \geq b$ and $\lambda_k \geq 0$.
Note that (II) can be simply bounded as follows. Due to the dual update rule, for any $\lambda \in [0, \frac{2}{\slater}],$
\begin{align*}
    (\lambda_{k+1} - \lambda)^2 
    &\leq \left(\lambda_k + \eta \left(\sum_{h=1}^H (b-g(s_h^k,a_h^k))\right) - \lambda\right)^2 \\
    &= (\lambda_k - \lambda)^2 + \eta^2\left(\sum_{h=1}^H (b-g(s_h^k,a_h^k))\right)^2 + 2\eta(\lambda_k-\lambda)\left(\sum_{h=1}^H (b-g(s_h^k,a_h^k))\right).
\end{align*}
where the first inequality is due to the fact that $[\cdot]_{[0,\frac{2}{\slater}]}$ is a contraction operator. By rearranging it and summing over $k=1,\ldots,K$, we have
\begin{align}\label{eq:decomp (II)}
\begin{aligned}
    \textnormal{(II)}
    &\leq \frac{\lambda^2}{2\eta} + \frac{\eta \sum_{k=1}^K\left(\sum_{h=1}^H (b-g(s_h^k,a_h^k))\right)^2}{2}\\
    &\leq \frac{\lambda^2}{2\eta} + \frac{\eta KH^2}{2}.
\end{aligned}
\end{align}
where the second inequality follows from the fact that $|b - g(s_h^k,a_h^k)| \leq 1$. Next, we focus on bounding (I). Note that our design of $Q_h^k$ satisfies that
\begin{align}\label{eq:Q_h^k uppder bound}
    Q_h^k(s_h^k,a_h^k) \leq r(s_h^k,a_h^k) + \lambda_k  g(s_h^k,a_h^k) + \phi(s_h^k,a_h^k)^\top\bfw_{h+1}^k + \min_{s'}V_{h+1}^k(s') +  \beta\|\phi(s_h^k,a_h^k)\|_{\bmSigma_k^{-1}}.
\end{align}
Let $V_{r,h}^{*},V_{g,h}^{*}$ denote the value functions with respect to $\pi^{*}$. Then (I) can be bounded as follows.
\begin{align*}
    \textnormal{(I)}&=\sum_{h=1}^H \left(J_r^{*}(s_1)  - r(s_h^k,a_h^k)\right) +  \sum_{h=1}^H \lambda_k  \left(J_g^{*}(s_1)- g(s_h^k,a_h^k)\right) \\
    &\leq \sum_{h=1}^H \left(J_r^{*}(s_1) + \lambda_k  J_g^{*}(s_1) - Q_h^k(s_h^k,a_h^k) +\phi(s_h^k,a_h^k)^\top\bfw_{h+1}^k+ \min_{s'}V_{h+1}^k(s') + \beta\|\phi(s_h^k,a_h^k)\|_{\bmSigma_k^{-1}} \right) \\
    &\leq HJ_r^{*}(s_1) - V_{r,1}^{*}(s_1^k) + \lambda_k  \left(HJ_g^{*}(s_1) - V_{g,1}^{*}(s_1^k)\right) \\
    &\quad+ V_{r,1}^{*}(s_1^k) + \lambda_k  V_{g,1}^{*}(s_1^k)+\sum_{h=1}^H \left( - Q_h^k(s_h^k,a_h^k) +  PV_{h+1}^k(s_h^k,a_h^k)\right) + \sum_{h=1}^H 2\beta\|\phi(s_h^k,a_h^k)\|_{\bmSigma_k^{-1}}
\end{align*}
where the first inequality follows from \eqref{eq:Q_h^k uppder bound}, and the last inequality is because \Cref{lem:good event} implies that $\phi(s_h^k,a_h^k)^\top\bfw_{h+1}^k + \min_{s'} V_{h+1}^k(s') \leq PV_{h+1}^k(s_h^k,a_h^k) + \beta\|\phi(s_h^k,a_h^k)\|_{\bmSigma_k^{-1}}$. Note that $J_r^*,J_g^*$ are assumed to be constant across $s\in \calS$ by \Cref{assum:zeta}. Then, by \Cref{lem:chen2022 lemma 4}, we have for all $h\in [H]$,
\begin{align}\label{eq:chen2022 lemma 4}
    |(H-h+1)J_r^{*} - V_{r,h}^{*}(s)| \leq \spn(v_r^*), \ |(H-h+1)J_g^{*} - V_{g,h}^{*}(s)| \leq \spn(v_g^*).
\end{align}
By taking $h=1$ for \eqref{eq:chen2022 lemma 4}, it leads to 
\begin{align*}
    &HJ_r^{*}(s_1) - V_{r,1}^{*}(s_1^k) + \lambda_k  \left(HJ_g^{*}(s_1) - V_{g,1}^{*}(s_1^k)\right) \\
    &\quad+ V_{r,1}^{*}(s_1^k) + \lambda_k  V_{g,1}^{*}(s_1^k)+\sum_{h=1}^H \left( - Q_h^k(s_h^k,a_h^k) +  P  V_{h+1}^k(s_h^k,a_h^k)\right) + \sum_{h=1}^H 2\beta\|\phi(s_h^k,a_h^k)\|_{\bmSigma_k^{-1}}\\
    &\leq \spn(v_r^*) + \lambda_k  \spn(v_g^*) \\
    &\quad+ V_{r,1}^{*}(s_1^k) + \lambda_k  V_{g,1}^{*}(s_1^k) + \sum_{h=1}^H \left(- Q_h^k(s_h^k,a_h^k) +  P  V_{h+1}^k(s_h^k,a_h^k)\right) + \sum_{h=1}^H 2\beta\|\phi(s_h^k,a_h^k)\|_{\bmSigma_k^{-1}} \\
    &\leq \spn(v_r^*) + \frac{2}{\slater}  \spn(v_g^*) \\
    &\quad+ \underbrace{V_{r,1}^{*}(s_1^k) + \lambda_k  V_{g,1}^{*}(s_1^k) + \sum_{h=1}^H \left(- Q_h^k(s_h^k,a_h^k) +  P  V_{h+1}^k(s_h^k,a_h^k)\right)}_{\textnormal{(III)}} + \sum_{h=1}^H 2\beta\|\phi(s_h^k,a_h^k)\|_{\bmSigma_k^{-1}}.
\end{align*}
where the last inequality follows from $\lambda_k \leq \frac{2}{\slater}$ for all $k$. Term (III) can be further bounded as follows.
\begin{align*}
    \textnormal{(III)} &= V_{r,1}^{*}(s_1^k) + \lambda_k  V_{g,1}^{*}(s_1^k) + \sum_{h=1}^H \left(- Q_h^k(s_h^k,a_h^k) +  P  V_{h+1}^k(s_h^k,a_h^k)\right)\\
    &= V_{r,1}^{*}(s_1^k) + \lambda_k  V_{g,1}^{*}(s_1^k) - V_1^k(s_1^k) + V_1^k(s_1^k) + \sum_{h=1}^H \left(- Q_h^k(s_h^k,a_h^k) +  P  V_{h+1}^k(s_h^k,a_h^k)\right) \\
    &= V_{r,1}^{*}(s_1^k) + \lambda_k  V_{g,1}^{*}(s_1^k) - V_1^k(s_1^k)\\
    &\quad+V_1^k(s_1^k) - Q_1^k(s_1^k,a_1^k) +  P V_{2}^k(s_1^k,a_1^k) - V_2^k(s_2^k)+ V_2^k(s_2^k)  + \sum_{h=2}^H \left(-Q_h^k(s_h^k,a_h^k) +  P  V_{h+1}^k (s_h^k,a_h^k)\right) \\
    &\quad \vdots\\
    &= V_{r,1}^{*}(s_1^k) + \lambda_k  V_{g,1}^{*}(s_1^k) - V_1^k(s_1^k)+\sum_{h=1}^H \left(V_h^k(s_h^k) - Q_h^k(s_h^k,a_h^k)\right) + \sum_{h=1}^H \left( P V_{h+1}^k(s_h^k,a_h^k) - V_{h+1}^k(s_{h+1}^k)\right)
\end{align*}
This implies that for each $k \in [K]$,
\begin{align}\label{eq:decomp (I)}
\begin{aligned}
    \textnormal{(I)} 
    &\leq \spn(v_r^*) + \frac{2}{\slater}  \spn(v_g^*) \\
    &\quad+ V_{r,1}^{*}(s_1^k) + \lambda_k  V_{g,1}^{*}(s_1^k) - V_1^k(s_1^k)+\sum_{h=1}^H \left(V_h^k(s_h^k) - Q_h^k(s_h^k,a_h^k)\right) + \sum_{h=1}^H \left( P V_{h+1}^k(s_h^k,a_h^k) - V_{h+1}^k(s_{h+1}^k)\right) \\
    &\quad+ \sum_{h=1}^H 2\beta\|\phi(s_h^k,a_h^k)\|_{\bmSigma_k^{-1}}.
\end{aligned}    
\end{align}
Note that taking sum over $k=1,\ldots,K$ to \eqref{eq:decomp (I)} yields the first term in \eqref{eq:decomp main}. In addition to this, by applying \eqref{eq:decomp (II)}, we have for any $\lambda \in [0, \frac{2}{\slater}]$,
\begin{align*}
    &\sum_{t=1}^T \left(J_r^*(s_1)  - r(s_t,a_t)\right) + \lambda \sum_{t=1}^T (b-g(s_t,a_t)) \\
    &\leq K \left(\spn(v_r^*) + \frac{2}{\slater}\spn(v_g^*)\right) + \sum_{k=1}^K\sum_{h=1}^H 2\beta\|\phi(s_h^k,a_h^k)\|_{\bmSigma_k^{-1}} \\
    &\quad+  \sum_{k=1}^K\left(V_{r,1}^{*}(s_1^k) + \lambda_k  V_{g,1}^{*}(s_1^k) - V_1^k(s_1^k)\right)\\
    &\quad+\sum_{k=1}^K \sum_{h=1}^H \left(V_h^k(s_h^k) - Q_h^k(s_h^k,a_h^k)\right) + \sum_{k=1}^K\sum_{h=1}^H \left( P V_{h+1}^k(s_h^k,a_h^k) - V_{h+1}^k(s_{h+1}^k)\right)\\
    &\quad+\frac{\lambda^2}{2\eta} + \frac{\eta KH^2}{2}.
\end{align*}
Note that $V_h^k(s_h^k) \leq \widetilde V_h^k(s_h^k) = Q_h^k(s_h^k,a_h^k)$. It follows that $\sum_{k=1}^K \sum_{h=1}^H \left(V_h^k(s_h^k) - Q_h^k(s_h^k,a_h^k)\right) \leq 0$. This concludes the proof.

\subsection{Proof of Lemma~\ref{lem:Regret + lambda Volation}}
Assume that the statement of \Cref{lem:good event} holds, which occurs with probability at least $1-\delta$. By \eqref{eq:decomp}, we have
\begin{align}\label{eq:Regret + lambda Volation decomp}
\begin{aligned}
    \Regret(T) + \lambda \Violation(T)
    &\leq K\zeta+ \underbrace{\sum_{k=1}^K\sum_{h=1}^H 2\beta\|\phi(s_h^k,a_h^k)\|_{\bmSigma_k^{-1}}}_{\textnormal{(I)}} \\
    &\quad+  \underbrace{\sum_{k=1}^K\left(V_{r,1}^{*}(s_1^k) + \lambda_k  V_{g,1}^{*}(s_1^k) - V_1^k(s_1^k)\right)}_{\textnormal{(II)}}\\
    &\quad+ \underbrace{\sum_{k=1}^K\sum_{h=1}^H \left( P V_{h+1}^k(s_h^k,a_h^k) - V_{h+1}^k(s_{h+1}^k)\right)}_{\textnormal{(III)}} +\frac{\lambda^2}{2\eta} + \frac{\eta KH^2}{2}.
\end{aligned}
\end{align}

By \Cref{lem:optimism}, we have 
\begin{equation}\label{eq:Regret + lambda Volation (II)}
    \textnormal{(II)}\leq 0.
\end{equation}
Moreover, since (III) is the sum of a martingale difference sequence, it can be bounded using the Azuma-Hoeffding inequality. Note that $\spn(V_{h+1}^k) \leq 2\spn(v_r^*) + \frac{4}{\slater}\spn(v_g^*) = 2\zeta$. Then, with probability at least $1-\delta$,
\begin{align}\label{eq:Regret + lambda Volation (III)}
    \textnormal{(III)} \leq 2\zeta\sqrt{2KH\log(2/\delta)}.
\end{align}
Now, we bound (I). The standard elliptical potential lemma cannot be directly applied due to the definition of $\bmSigma_k$, i.e., $\bmSigma_{k+1} \neq \bmSigma_k + \phi(s_h^k,a_h^k)\phi(s_h^k,a_h^k)^\top$. Therefore, our proof uses a similar technique from the proof of Theorem 5 in \citet{wei2021learning}, which decomposes the desired sum into two terms as follows.
\begin{align*}
    &\sum_{k=1}^K\sum_{h=1}^H\|\phi(s_h^k,a_h^k)\|_{\bmSigma_k^{-1}} \\
    &= \sum_{k: 2\det(\bmSigma_k) \geq \det(\bmSigma_{k+1})}\sum_{h=1}^H\|\phi(s_h^k,a_h^k)\|_{\bmSigma_k^{-1}}  + \sum_{k: 2\det(\bmSigma_k) < \det(\bmSigma_{k+1})}\sum_{h=1}^H\|\phi(s_h^k,a_h^k)\|_{\bmSigma_k^{-1}}.
\end{align*}
Note that 
\begin{align*}
    \sum_{k: 2\det(\bmSigma_k) \geq \det(\bmSigma_{k+1})}\sum_{h=1}^H\|\phi(s_h^k,a_h^k)\|_{\bmSigma_k^{-1}} 
    &\leq \sqrt{2}\sum_{k: 2\det(\bmSigma_k) \geq \det(\bmSigma_{k+1})}\sum_{h=1}^H\|\phi(s_h^k,a_h^k)\|_{\bmSigma_{k+1}^{-1}} \\
    &\leq \sqrt{2T}\sqrt{\sum_{k=1}^K\sum_{h=1}^H\|\phi(s_h^k,a_h^k)\|^2_{\bmSigma_{k+1}^{-1}}}.
\end{align*}
Since $\bmSigma_{k+1} = I + \sum_{\tau=1}^k \sum_{h=1}^H \phi(s_h^\tau,a_h^\tau)\phi(s_h^\tau,a_h^\tau)^\top \succeq \bar\bmSigma_{k,h}$, where $\bar\bmSigma_{k,h}$ is defined as $\bar\bmSigma_{k,h} = I + \sum_{\tau=1}^{k-1}\sum_{h'=1}^H \phi(s_{h'}^\tau, a_{h'}^\tau)\phi(s_{h'}^\tau, a_{h'}^\tau)^\top + \sum_{h''=1}^h \phi(s_{h''}^k,a_{h''}^k)\phi(s_{h''}^k,a_{h''}^k)^\top$. Then it follows that
\begin{align*}
    \sum_{k=1}^K\sum_{h=1}^H\|\phi(s_h^k,a_h^k)\|^2_{\bmSigma_{k+1}^{-1}} \leq \sum_{k=1}^K\sum_{h=1}^H\|\phi(s_h^k,a_h^k)\|^2_{\bar\bmSigma_{k,h}^{-1}} \leq 2\log\frac{\det(\bmSigma_{K+1})}{\det(\bmSigma_1)} \leq 2d\log(T+1).
\end{align*}
where the first inequality follows from that $\|\cdot\|_{\bmSigma_{k+1}^{-1}} \leq \|\cdot\|_{\bar\bmSigma_{k,h}^{-1}}$, the second inequality follows from \Cref{lem:elliptical potential}, and the last inequality follows from that $\bmSigma_1 = I$ and
\begin{align*}
    \det(\bmSigma_{K+1}) \leq \|\bmSigma_{K+1}\|_2^d \leq \left(1 + \sum_{k,h}\|\phi(s_h^k,a_h^k)\|_2\right)^d \leq (T+1)^d.
\end{align*}
It follows that
\begin{align*}
    \sum_{k: 2\det(\bmSigma_k) \geq \det(\bmSigma_{k+1})}\sum_{h=1}^H\|\phi(s_h^k,a_h^k)\|_{\bmSigma_k^{-1}}  \leq 2\sqrt{dT\log(T+1)}.
\end{align*}
Moreover, by \Cref{lem:doubling trick boundt}, $|\{k\in [K]: 2\det(\bmSigma_k) < \det(\bmSigma_{k+1})\}| \leq \frac{3}{2}d\log (1+T)$. Since $\|\phi(s_h^k,a_h^k)\|_{\bmSigma_k^{-1}} \leq 1$, it follows that
\begin{align*}
    \sum_{k: 2\det(\bmSigma_k) < \det(\bmSigma_{k+1})}\sum_{h=1}^H\|\phi(s_h^k,a_h^k)\|_{\bmSigma_k^{-1}} \leq \frac{3}{2}dH\log (1+T).
\end{align*}
Then we have
\begin{align*}
    \sum_{k=1}^K\sum_{h=1}^H\|\phi(s_h^k,a_h^k)\|_{\bmSigma_k^{-1}} \leq 2\sqrt{dT\log(T+1)} + \frac{3}{2}dH\log (1+T).
\end{align*}
Thus, we have
\begin{align}\label{eq:Regret + lambda Volation (I)}
    \textnormal{(I)} \leq 4\beta\sqrt{dT\log(T+1)} + 3\beta dH\log (1+T).
\end{align}
By applying \eqref{eq:Regret + lambda Volation (II)}, \eqref{eq:Regret + lambda Volation (III)}, and \eqref{eq:Regret + lambda Volation (I)} to \eqref{eq:Regret + lambda Volation decomp}, for all $\lambda \in [0, \frac{2}{\gamma}]$,
\begin{align*}
    \Regret(T) + \lambda \Violation(T) 
    &\leq 1 + K\zeta + 4\beta\sqrt{dT\log(T+1)} + 3\beta dH\log (1+T) \\
    &\quad+ 2\zeta\sqrt{2KH\log(2/\delta)} + \frac{\lambda^2}{2\eta} + \frac{\eta KH^2}{2}\\
    &\leq 1 + K\zeta + 4\beta\sqrt{dT\log(T+1)} + 3\beta dH\log (1+T) \\
    &\quad+ 2\zeta\sqrt{2KH\log(2/\delta)} + \frac{(2/\gamma)^2}{2\eta} + \frac{\eta KH^2}{2}.
\end{align*}
Recall that $K = T^{2/3}$, $H = T^{1/3}$, $\beta = \bigO((\zeta+1)d)$, $\eta = \frac{2}{\gamma\sqrt{KH^2}} $. It follows that for all $\lambda \in [0,\frac{2}{\gamma}]$
\[
    \Regret(T) + \lambda \Violation(T) = \bigO\left(\left(\zeta + \frac{2}{\gamma}\right) T^{2/3} + (\zeta+1)(\sqrt{d^3 T} + d^2T^{1/3})\right).
\]

\subsection{Proof of Lemma~\ref{lem:regret lower bound}}
By \Cref{thm:strong duality} and \Cref{lem:optimal dual variable bound}, there exists a bounded optimal dual variable $\lambda^*$. Given $\lambda^*$, consider a reward function $r +\lambda^*(g-b):\calS\times\calA \to \bbR$ such that $[r +\lambda^*(g-b)](s,a) = r(s,a) + \lambda^*(g(s,a)-b)$ for each $(s,a) \in \calS\times\calA$. Then, by the Bellman optimality condition for weakly communicating unconstrained MDPs, there exists $J_{\lambda^*}^*(s_1) \in \bbR$, $v_{\lambda^*}^*:\calS \to \bbR$ such that for any $(s,a) \in \calS \times\calA$,
\begin{align}\label{eq:lem:regret lower bound 1}
    r(s,a)+\lambda^*(g(s,a)-b) + Pv_{\lambda^*}^*(s,a) - J_{\lambda^*}^*(s_1) \leq v_{\lambda^*}^*(s).
\end{align}
Next, we show that $J_r^*(s_1) = J_{\lambda^*}^*(s_1)$. Note that
\begin{align*}
    D(\lambda^*) 
    &= \sup_{\pi\in\PiSR} J_r^\pi(s_1) + \lambda^*(J_g^\pi(s_1) - b) \\
    &= \sup_{\pi\in\PiSR} J_{r+\lambda^* (g-b)}^{\pi}(s_1) \\
    &= \max_{\pi\in\PiSR} J_{r+\lambda^* (g-b)}^{\pi}(s_1) \\
    &= J_{\lambda^*}^*(s_1)
\end{align*}
where the second equality follows from that for any stationary policy $J_{r+\lambda^*(g-b)}^\pi(s_1) = (r+\lambda^*(g - b\bm{1}))^\top \occ{\pi}{s_1} = r^\top \occ{\pi}{s_1}+\lambda^*(g^\top \occ{\pi}{s_1} -b) = J_r^\pi(s_1) + \lambda^*(J_g^\pi(s_1) -b)$ with the all-1 vector $\bm{1}\in \bbR^{|\calS|\times|\calA|}$, and the third equality is due to the fact that an optimal stationary policy is attainable for weakly communicating MDPs. Moreover, by \Cref{thm:strong duality}, we have $J_r^*(s_1) = D(\lambda^*)$. Therefore, it follows that $J_r^*(s_1) = J_{\lambda^*}^*(s_1)$ as desired. Applying this to \eqref{eq:lem:regret lower bound 1}, it can be written as
\[
    r(s,a)+\lambda^*(g(s,a)-b) + Pv_{\lambda^*}^*(s,a) - J_{r}^*(s_1) \leq v_{\lambda^*}^*(s).
\]
Given the trajectory $\{(s_t,a_t)\}_{t=1}^T$ generated by our algorithm, taking $s=s_t,a=a_t$ yields
\begin{align*}
    r(s_t,a_t) - J_{r}^*(s_1) 
    &\leq \lambda^*(b-g(s_t,a_t)) + v_{\lambda^*}^*(s_{t}) - Pv_{\lambda^*}^*(s_t,a_t)\\
    &= \lambda^*(b-g(s_t,a_t)) + v_{\lambda^*}^*(s_{t+1}) - Pv_{\lambda^*}^*(s_t,a_t) + v_{\lambda^*}^*(s_t) - v_{\lambda^*}^*(s_{t+1})
\end{align*}
Summing over $t=1,\ldots,T$ yields
\begin{align*}
    \sum_{t=1}^T (r(s_t,a_t) - J_{r}^*(s_1)) 
    &\leq \lambda^* \sum_{t=1}^T (b-g(s_t,a_t)) \\
    &\quad+ \underbrace{\sum_{t=1}^Tv_{\lambda^*}^*(s_{t+1}) - Pv_{\lambda^*}^*(s_t,a_t)}_{\textnormal{(I)}} + \underbrace{\sum_{t=1}^Tv_{\lambda^*}^*(s_t) - v_{\lambda^*}^*(s_{t+1})}_{\textnormal{(II)}}.
\end{align*}
For (I), since it is a martingale difference sequence, by the Azuma-Hoeffding inequality, we have 
$\textnormal{(I)}\leq \spn(v_{\lambda^*}^*)\sqrt{T\log(1/\delta)}$ with probability at least $1-\delta$. For (II), since it is a telescoping sum, we have $\textnormal{(II)}\leq v_{\lambda^*}^*(s_1) - v_{\lambda^*}^*(s_{T+1}) \leq \spn(v_{\lambda^*}^*)$. Multiplying both sides by $(-1)$, with probability at least $1-\delta$,
\begin{align*}
    \sum_{t=1}^T (J_r^*(s_1)-r(s_t,a_t)) \geq -\lambda^* \sum_{t=1}^T (b-g(s_t,a_t)) - \spn({v_{\lambda^*}^*})\sqrt{2T\log(1/\delta)} - \spn({v_{\lambda^*}^*}).
\end{align*}
Since $\Regret(T) = \sum_{t=1}^T (J_r^*(s_1) - r(s_t,a_t))$ and $\Violation(T) = \sum_{t=1}^T (b - g(s_t,a_t))$, this completes the proof.

\subsection{Proof of Theorem~\ref{thm:regret analysis}}
Suppose that \Cref{lem:good event} holds, which occurs with probability at least $1-\delta$. Then, by \Cref{lem:Regret + lambda Volation}, for any $\lambda \in [0,\frac{2}{\gamma}]$, with probability at least $1-\delta$,
\begin{align}\label{eq:thm2 0}
    \sum_{t=1}^T \left(J_r^*  - r(s_t,a_t)\right) + \lambda \sum_{t=1}^T (b-g(s_t,a_t)) =\bigO\left(\left(\zeta + \frac{2}{\gamma}\right) T^{2/3} + (\zeta+1)(\sqrt{d^3 T} + d^2T^{1/3})\right).
\end{align}
By taking $\lambda=0$,
\begin{align*}
    \sum_{t=1}^T \left(J_r^*  - r(s_t,a_t)\right) = \bigO\left(\left(\zeta + \frac{2}{\gamma}\right) T^{2/3} + (\zeta+1)(\sqrt{d^3 T}+ d^2T^{1/3})\right).
\end{align*}
If $\sum_{t=1}^T (b-g(s_t,a_t)) \leq 0$, then the proof is finished. Otherwise, we show a violation upper bound as follows. With probability at least $1-\delta$, by \Cref{lem:regret lower bound},
\begin{align}\label{eq:thm2 1}
\begin{aligned}
    &\sum_{t=1}^T \left(J_r^*  - r(s_t,a_t)\right) + \frac{2}{\gamma} \sum_{t=1}^T (b-g(s_t,a_t)) \\
    &\geq \left(\frac{2}{\gamma}-\lambda^*\right) \sum_{t=1}^T (b-g(s_t,a_t)) - \spn({v_{\lambda^*}^*})\sqrt{T\log(1/\delta)} - \spn({v_{\lambda^*}^*})\\
    &\geq \frac{1}{\gamma} \sum_{t=1}^T (b-g(s_t,a_t)) - \spn({v_{\lambda^*}^*})\sqrt{T\log(1/\delta)} - \spn({v_{\lambda^*}^*})
\end{aligned}
\end{align}
where the second inequality follows from $\lambda^* \leq 1/\slater$ (\Cref{lem:optimal dual variable bound}) and taking $\lambda=\frac{2}{\slater}$. By union bound, \Cref{lem:good event}, \eqref{eq:thm2 0}, and \eqref{eq:thm2 1} hold with probability at least $1-3\delta$. Finally, with probability at least $1-3\delta$, we have
\begin{align*}
    \sum_{t=1}^T(b-g(s_t,a_t)) = \bigO\left(\left(\gamma\zeta + 2\right) T^{2/3} + \gamma(\zeta+1)(\sqrt{d^3 T} + d^2T^{1/3} )+\gamma\spn({v_{\lambda^*}^*})\sqrt{T}\right).
\end{align*}

\section{Numerical experiments}\label{appendix:numerical}
We evaluate \Cref{alg:main} on a tabular CMDP, constructed by modifying the MDP of \citet{he2022near} to incorporate constraints and an infinite-horizon setting. The state space is $\mathcal{S}=\{0,\ldots,H+1\}$, and the action space is $\mathcal{A}=\{-1,1\}^{m-1}$. For each state $s \in \{0,\ldots,H-1\}$ and action $a \in \mathcal{A}$, the reward and constraint functions are defined as $r(s,a)=0.4 \cdot \mathrm{score}(a)$ and $g(s,a)=\mathrm{score}(a)$, where $\mathrm{score}(a)$ denotes the fraction of positive components in $a$. State $H$ yields reward 0 and state $H+1$ yields reward 1; both states incur zero constraint value. For each $s \in \{0, \cdots, H-1\}$, the process transitions to state $s+1$ with probability $0.95- 0.01 \cdot \mathbf{1}_{m-1}^\top a$ and transitions to state $H+1$ otherwise. The terminal/reset states $H$ and $H+1$ transition back to the initial state, thereby inducing an infinite-horizon process.

In the experiment, we choose $H=6, m=2$, $d=|\calS||\calA| = 12$, and the constraint threshold to $b=0.7$. We compare \Cref{alg:main} with a random agent and Algorithm 2 of \citet{ghosh2023achieving}, which is a canonical primal-dual algorithm based on a finite-horizon approximation. For \Cref{alg:main}, we set the hyperparameters according to the choices specified in \Cref{thm:regret analysis}, while we use the adaptive clipping parameter $\zeta_k= \spn(v_r^*) + \lambda_k\spn(v_g^*)$ and set $\beta_k=(\zeta_k+1)d$ to improve learning stability. For Algorithm 2 in \citet{ghosh2023achieving}, we follow the default parameter choices, except that we choose $H=T^{1/4}$, omitting the factor $d^{3/4}$ in the denominator, and set $\beta=dH$, omitting the constant and logarithmic terms. We conduct 4 simulations with $K=1000$ under different random seeds.

As shown in \Cref{fig:exp}, \Cref{alg:main} achieves sublinear regret and constraint violation, providing empirical support for the main claim of \Cref{thm:regret analysis}. In contrast, Algorithm 2 of \citet{ghosh2023achieving} exhibits near-random performance. This behavior can be attributed to the choice of $\beta$, which induces a large bonus term and consequently leads to overly optimistic value-function estimates. These results suggest that the clipping mechanism in \Cref{alg:main}, characterized by the clipping parameter $\zeta$ and its dependence on $\spn(v_r^*)$ and $\spn(v_g^*)$, provides improved adaptivity in this CMDP. Overall, the results highlight the empirical efficiency of \Cref{alg:main}.

\begin{figure}[hbt]
\begin{subfigure}{0.45\linewidth}
\centering
\includegraphics[width=\linewidth]{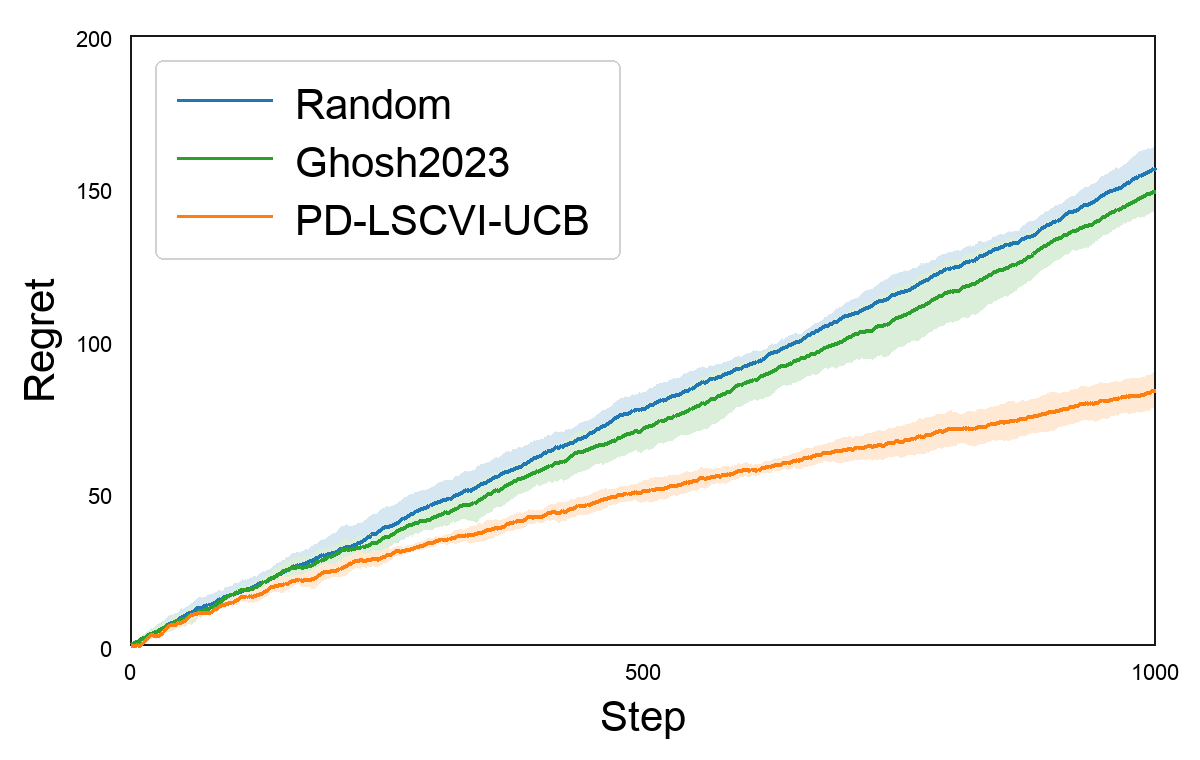}
\end{subfigure}
\hfill
\begin{subfigure}{0.45\linewidth}
\centering
\includegraphics[width=\linewidth]{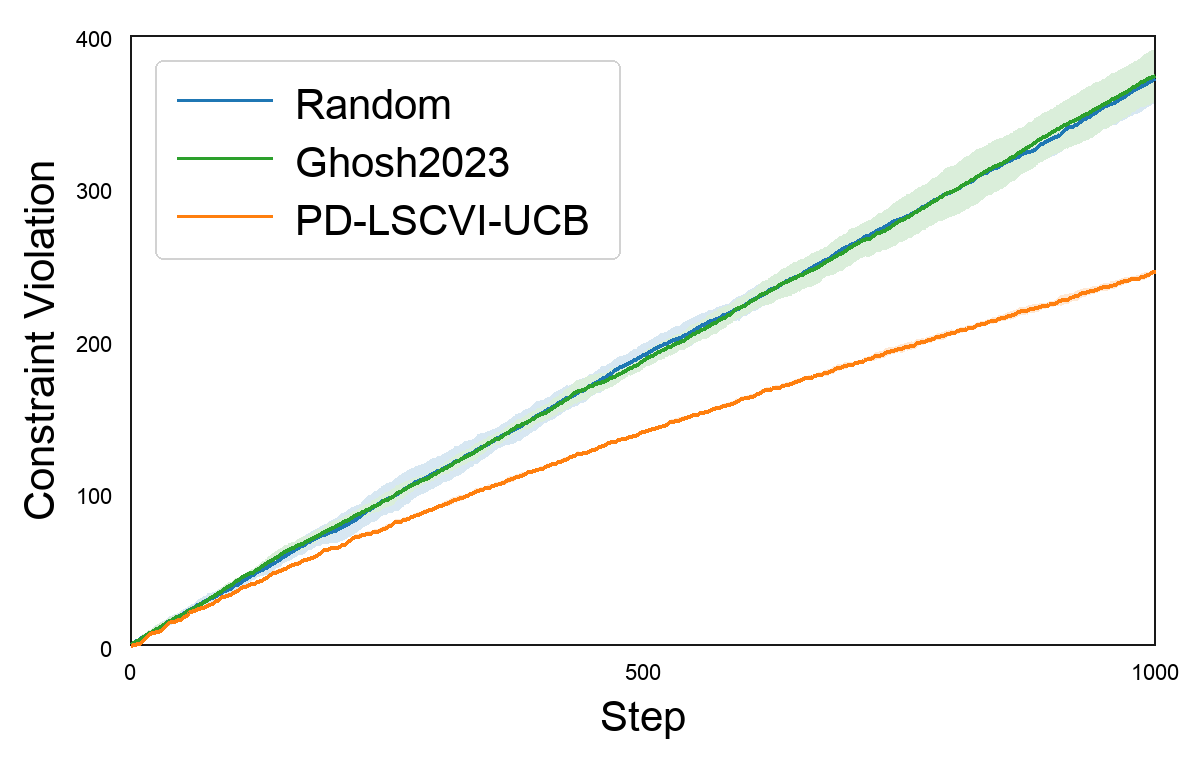}
\end{subfigure}
\caption{Comparison of random agent, Algorithm 2 in \citet{ghosh2023achieving}, and PD-LSCVI-UCB (\Cref{alg:main}). Each plot displays regret and constraint violation for $K=1000$ episodes. The shaded regions indicate 95\% confidence intervals.}
\label{fig:exp}
\end{figure}

\section{Auxiliary lemmas}
\begin{lemma}[Lemma D.1 in \citet{jin2020provably}]\label{lem:jin2020 lemma D.1}
Let $\bmSigma_t = I + \sum_{i=1}^t \phi_i\phi_i^\top$ where $\phi_i \in \bbR^d$. Then $\sum_{i=1}^t \phi_i^\top \bmSigma_t^{-1}\phi_i \leq d$.    
\end{lemma}

\begin{lemma}[Lemma 12 in \citet{hong2024reinforcement}]\label{lem:hong2024 lemma 12}
    Let $V:\calS \to [0,B]$ be a bounded function. Then, there exists $\bfw^* \in \bbR^d$ such that $PV(s,a) = \phi(s,a)^\top \bfw^*$ for all $(s,a) \in \calS\times \calA$ and $\|\bfw^*\|_2 \leq B\sqrt{d}$.
\end{lemma}

\begin{lemma}[Lemma 13 in \citet{hong2024reinforcement}]\label{lem:hong2024 lemma 13}
    Let $\bfw$ be a ridge regression coefficient obtained by regressing $y \in [0,B]$ on $\bfx \in \bbR^d$ using the dataset $\{(\bfx_i,y_i)\}_{i=1}^n$ so that $\bfw = \bmSigma^{-1} \sum_{i=1}^n \bfx_i y_i$, where $\bmSigma = I+\sum_{i=1}^n \bfx_i\bfx_i^\top$. Then, $\|\bfw\|_2 \leq B\sqrt{dn}$.
\end{lemma}

\begin{lemma}[Lemma 14 in \citet{hong2024reinforcement}] \label{lem:concentration}
    Let $\{x_t\}_{t=1}^\infty$ be a stochastic process on state space $\calS$ with corresponding filtration $\{\calF_t\}_{t=0}^\infty$. Let $\{\phi_t\}_{t=0}^\infty$ be a $\bbR^d$-valued stochastic process where $\phi_t \in \calF_{t-1}$ and $\|\phi_t\|_2 \leq 1$. Let $\bmSigma_n = I + \sum_{t=1}^n \phi_t\phi_t^\top$. Then for any $\delta > 0$ and any given function class $\calV$, with probability at least $1-\delta$, for all $n\geq 0$, and any $V \in \calV$ satisfying $\spn(V) \leq 2\zeta$, we have
    \begin{align*}
        \left\|\sum_{t=1}^n \phi_t(V(x_t) - \bbE[V(x_t) | \calF_{t-1}])\right\|_{\bmSigma_n^{-1}}^2 \leq 16\zeta^2 \left[\frac{d}{2}\log(n+1) + \log\frac{\calN_\epsilon}{\delta}\right] + 8n^2\epsilon^2
    \end{align*}
    where $\calN_\epsilon$ is the $\epsilon$-covering number of $\calV$ with respect to the distance $\textnormal{dist}(V,V') = \sup_x |V(x) - V'(x)|$.
\end{lemma}

\begin{lemma}[Lemma 15 in \citet{hong2024reinforcement}]\label{lem:hong2024 lemma 15}
Let $\calV$ be a class of functions mapping from $\calS$ to $\bbR$ with the following parametric form:
\begin{align*}
    V(\cdot) = \left(\max_a \bfw^\top \phi(\cdot,a) + v + \beta\|\phi(\cdot,a)\|_{\bmSigma^{-1}}\right) \wedge M
\end{align*}
where the parameters $(\bfw, \beta, v, \bmSigma, M)$ satisfy $\|\bfw\|_2 \leq L$, $\beta \in [0,B]$, $v\in [0,D]$, $M\geq 0$ and the minimum eigenvalue satisfies $\lambda_{\min}(\bmSigma)\geq 1$. Assume $\|\phi(s,a)\|_2 \leq 1$ for all $(s,a)$, and let $\calN_\epsilon$ be the $\epsilon$-covering number of $\calV$ with respect to $\textnormal{dist}(V,V') \leq \sup_s |V(s) - V'(s)|$. Then
\[
    \log(\calN_\epsilon) \leq d\log(1 + 8L/\epsilon) + \log(1+4D/\epsilon) + d^2\log\left(1 + 8d^{1/2}B^2/\epsilon^2\right).
\]
    
\end{lemma}

\begin{lemma}[Lemma D.2 in \citet{jin2020provably}]\label{lem:elliptical potential}
    Let $\{\phi_t\}_{t\geq 0}$ be a bounded sequence in $\bbR^d$ satisfying $\sup_{t\geq 0} \|\phi_t\|\leq 1$. Let $\bmSigma_0 \in \bbR^{d\times d}$ be a positive definite matrix. For any $t \geq 0$, we define $\bmSigma_t = \bmSigma_0 + \sum_{j=1}^t \phi_j \phi_j^\top$. Then if the smallest eigenvalue of $\bmSigma_0$ satisfies $\lambda_{\min}(\bmSigma_0) \geq 1$, we have
    \begin{align*}
        \log\frac{\det(\bmSigma_t)}{\det(\bmSigma_0)} \leq \sum_{j=1}^t \phi_j^\top \bmSigma_{j-1}^{-1} \phi_j \leq 2\log\frac{\det(\bmSigma_t)}{\det(\bmSigma_0)}.
    \end{align*}
\end{lemma}

\begin{lemma}\label{lem:doubling trick boundt}
    $|\{k\in [K]: 2\det(\bmSigma_k) < \det(\bmSigma_{k+1})\}| \leq \frac{3}{2} d\log(1+T).$
\end{lemma}
\begin{proof}
    For simplicity, let $N = |\{k\in [K]: 2\det(\bmSigma_k) < \det(\bmSigma_{k+1})\}|$. Note that since $\det(\bmSigma_{k+1}) \geq \det(\bmSigma_k)$, we have
    \begin{align*}
        \det(\bmSigma_{K+1}) = \prod_{k=1}^{K} \frac{\det(\bmSigma_{k+1})}{\det(\bmSigma_k)} \geq 2^N.
    \end{align*}
    Moreover,
    \begin{align*}
        \det(\bmSigma_{K+1}) = \det\left(I + \sum_{k,h}\phi(s_h^k,a_h^k)\phi(s_h^k,a_h^k)^\top\right) \leq (1 + T)^d.
    \end{align*}
    Finally, we have $N \leq d\log_2 (1+T) \leq \frac{3}{2} d\log(1+T)$.
\end{proof}





\end{document}